\pdfoutput=1

\documentclass[review]{elsarticle}

\makeatletter
\def\ps@pprintTitle{%
 \let\@oddhead\@empty
 \let\@evenhead\@empty
 \def\@oddfoot{\centerline{\thepage}}%
 \let\@evenfoot\@oddfoot}
\makeatother

\usepackage{algorithm}
\usepackage{algpseudocode}
\usepackage{xspace}
\usepackage[usenames,dvipsnames,table]{xcolor}
\usepackage{listings}
\usepackage{fancyvrb}
\usepackage{datetime}
\usepackage[breakall]{truncate}
\usepackage{graphicx}
\usepackage{tikz,pgfplots}
\usepackage{amsmath}
\usepackage{amssymb}
\usepackage{array}
\usepackage{amsthm}
\usepackage{mathtools}

\newtheorem{mprop}{Proposition}
\newtheorem{mtheor}{Theorem}
\usepackage[caption=false]{subfig}
\usepackage{hyperref}
\usepackage{url}
\usepackage{booktabs}
\usepackage{multirow}
\usepackage{pgfplotstable}
\usepackage{longtable}
\usepackage{changepage}
\usepackage{rotating}
\usepackage{calc}
\usepackage[group-four-digits,detect-weight=true]{siunitx}
\usepackage[inline,shortlabels]{enumitem}
\usepackage[capitalise]{cleveref}

\input{Definitions}

\journal{Neurocomputing}

\begin{document}

\renewcommand{\theaffn}{\arabic{affn}}

\begin{frontmatter}

\title{Modified Frank--Wolfe Algorithm for Enhanced Sparsity in Support Vector Machine Classifiers}

\author[AdUAM]{Carlos M. Ala\'{i}z}
\ead{carlos.alaiz@inv.uam.es}

\author[AdKUL]{Johan~A.K.~Suykens}
\ead{johan.suykens@esat.kuleuven.be}

\address[AdUAM]{Dpto. Ing. Inform\'atica, Universidad Aut\'onoma de Madrid, 28049 Madrid, Spain}
\address[AdKUL]{Dept. Electrical Engineering (ESAT--STADIUS), KU Leuven, B-3001 Leuven, Belgium}

\begin{abstract}
 This work proposes a new algorithm for training a re-weighted $\lt$ Support Vector Machine (SVM), inspired on the re-weighted Lasso algorithm of Cand\`es \emph{et al.} and on the equivalence between Lasso and SVM shown recently by Jaggi. In particular, the margin required for each training vector is set independently, defining a new weighted SVM model. These weights are selected to be binary, and they are automatically adapted during the training of the model, resulting in a variation of the Frank--Wolfe optimization algorithm with essentially the same computational complexity as the original algorithm.

 As shown experimentally, this algorithm is computationally cheaper to apply since it requires less iterations to converge, and it produces models with a sparser representation in terms of support vectors and which are more stable with respect to the selection of the regularization hyper-parameter.
\end{abstract}


\end{frontmatter}

\section{Introduction}
\label{SecIntro}

Regularization is an essential mechanism in Machine Learning that usually refers to the set of techniques that attempt to improve the estimates by biasing them away from their sample-based values towards values that are deemed to be more ``physically plausible''~\cite{Friedman1989}.
In practice, it is often used to avoid over-fitting, apply some prior knowledge about the problem at hand or induce some desirable properties over the resulting learning machine.
One of these properties is the so called sparsity, which can be roughly defined as expressing the learning machines using only a part of the training information. This has advantages in terms of the interpretability of the model and its manageability, and also preventing the over-fitting.
Two representatives of this type of models are the Support Vector Machines (SVM~\cite{Cortes1995}) and the Lasso model~\cite{Tibshirani1996}, based on inducing sparsity at two different levels.
On the one hand, the SVMs are sparse in their representation in terms of the training patterns, which means that the model is characterized only by a subsample of the original training dataset.
On the other hand, the Lasso models induce sparsity at the level of the features, in the sense that the model is defined only as a function of a subset of the inputs, hence performing an implicit feature selection.

Recently, Jaggi~\cite{Jaggi2014} showed an equivalence between the optimization problems corresponding to a classification $\lt$-SVM and a constrained regression Lasso.
As explored in this work, this connection can be useful to transfer ideas from one field to the other. In particular, and looking for sparser SVMs, in this paper the reweighted Lasso approach of Cand\`es \emph{et al.}~\cite{Candes2008} is taken as the basis to define first a weighted $\lt$-SVM, and then to propose a simple way of adjusting iteratively the weights that leads to a Modified Frank--Wolfe algorithm.
This adaptation of the weights does not add an additional cost to the algorithm. Moreover, as shown experimentally the proposed approach needs less iterations to converge than the standard Frank--Wolfe, and the resulting SVMs are sparser and much more robust with respect to changes in the regularization hyper-parameter, while retaining a comparable accuracy.

In summary, the contributions of this paper can be stated as follows:
\begin{enumerate}[(i)]
 \item The definition of a new weighted SVM model, inspired by the weighted Lasso and the connection between Lasso and SVM. This definition can be further extended to a re-weighted SVM, based on an iterative scheme to define the weights.
 \item The proposal of a modification of the Frank--Wolfe algorithm based on the re-weighting scheme to train the SVM. This algorithm results in a sparser SVM model, which coincides with the model obtained using a standard SVM training algorithm over only an automatically-selected subsample of the original training data.
 \item The numerical comparison of the proposed model with the standard SVM over a number of different datasets. These experiments show that the proposed algorithm requires less iterations while providing a sparser model which is also more stable against modifications of the regularization parameter.
\end{enumerate}

The remaining of the paper is organized in the following way. \Ref{SecPre} summarizes some results regarding the connection of SVM with Lasso. The weighted and re-weighted SVM are introduced in \ref{SecWSVM}, whereas the proposed modified Frank--Wolfe algorithm is presented in \ref{SecMFW}. The performance of this algorithm is tested through some numerical experiments in \ref{SecExp}, and \ref{SecConc} ends the paper with some conclusions and pointers to further work.

\subsubsection*{Notation}

$\npat$ denotes the number of training patterns, and $\ndim$ the number of dimensions. The data matrix is denoted by $\mx = \prn{\mxp{1}, \mxp{2}, \dotsc, \mxp{\npat}}^\tr \in \Rpd$, where each row correspond to the transpose of a different pattern $\mxp{i} \in \Rd$. The corresponding vector of targets is $\vy \in \Rp$, where $\vyi{i} \in \set{-1, +1}$ denotes the label of the $i$-th pattern. The identity matrix of dimension $\npat$ is denoted by $\iden{\npat} \in \Rpp$.
\section{Preliminaries}
\label{SecPre}

This section covers some preliminary results concerning the Support Vector Machine (SVM) formulation, its connection with the Lasso model, and the re-weighted Lasso algorithm, which are included since they are the basis of the proposed algorithm.

\subsection{SVM Formulation}

The following $\lt$-SVM classification model (this model is described for example in~\cite{Keerthi2000}), crucial in~\cite{Jaggi2014}, will be used as the starting point of this work:
\begin{equation*}
 \label[problem]{EqProbSVMP}
 \minpc{\vw, \rho, \xi}{\frac{1}{2} \nt{\vw}^2 - \rho + \frac{C}{2} \sum_{i = 1}^\npat {\xi_i^2}}{\vw^\tr \mzi{i} \ge \rho - \xi_i} \eeq{,}
\end{equation*}
where $\mzi{i} = \vyi{i} \mxp{i}$.
Straightforwardly, the corresponding Lagrangian dual problem can be expressed as:
\begin{equation}
 \label[problem]{EqProbSVMD}
 \minpc{\va \in \Rp}{\va^\tr \mkh \va}{0 \le \vai{i} \le 1 \sepcons \sum_{i=1}^{\npat} \vai{i} = 1} \eeq{,}
\end{equation}
where $\mkh = \mz \mz^\tr + \frac{1}{C} \iden{\npat}$.
A non-linear SVM can be considered simply by substituting $\mz \mz^\tr$ by the (labelled) kernel matrix $\mk \hadam \vy \vy^\tr$ (where $\hadam$ denotes the Hadamard or component-wise product).

It should be noticed that the feasible region of \ref{EqProbSVMD} is just the probability simplex, and the objective function is simply a quadratic term.

\subsection{Connection between Lasso and SVM}
\label{SecEquiv}

There exists an equivalence between the SVM dual \ref{EqProbSVMD} and the following problem, which corresponds to a constrained Lasso regression model:
\begin{equation}
 \label[problem]{EqProbLasso}
 \minpc{\vw \in \Rd}{\nt{\mx \vw - \vy}^2}{\no{\vw} \le 1} \eeq{,}
\end{equation}
where in this case the vector $\vy \in \Rp$ does not need to be binary. In particular, a problem of the form of \ref{EqProbSVMD} can be rewritten in the form of \ref{EqProbLasso} and vice-versa~\cite{Jaggi2014}.

This relation is only at the level of the optimization problem, which means that an $\lt$-SVM model can be trained using the same approach as for training the Lasso model and the other way around (as done in~\cite{Alaiz2015}), but it cannot be extended to a prediction phase, since the Lasso model is solving a regression problem, whereas the SVM solves a classification one. Moreover, the number of dimensions and the number of patterns flip when transforming one problem into the other.
Nevertheless, and as illustrated in this paper, this connection can be valuable by itself to inspire new ideas.

\subsection{Re-Weighted Lasso}

The re-weighted Lasso (\rwla{}) was proposed as an approach to approximate the \lz{} norm by using the \lo{} norm and a re-weighting of the coefficients~\cite{Candes2008}.
In particular, this approach was initially designed to approximate the problem
\begin{equation*}
 \label[problem]{EqProbRWA}
 \minpc{\vw \in \Rd}{\nz{\vw}}{\vy = \mx \vw} \eeq{,}
\end{equation*}
by minimizing weighted problems of the form:
\begin{equation}
 \label[problem]{EqProbRWB}
 \minpc{\vw \in \Rd}{\sum_{i = 1}^{\ndim} \vti{i} \abs{\vwi{i}}}{\vy = \mx \vw} \eeq{,}
\end{equation}
for certain weights $\vti{i} > 0$, $i = 1, \dotsc, \ndim$. An iterative approach was proposed, where the previous coefficients are used to define the weights at the current iterate:
\begin{equation}
\label{EqRWLWeights}
 \vtik{i}{k} = \frac{1}{\abs{\vwik{i}{k-1}} + \epsilon} \eeq{,}
\end{equation}
what results in the following problem at iteration $k$:
\begin{equation*}
 \label[problem]{EqProbRWC}
 \minpc{\vwk{k} \in \Rd}{\sum_{i = 1}^{\ndim} \frac{1}{\abs{\vwik{i}{k-1}} + \epsilon} \abs{\vwik{i}{k}}}{\vy = \mx \vwk{k}} \eeq{.}
\end{equation*}
The idea is that if a coefficient is small, then it could correspond to zero in the ground-truth model, and hence it should be pushed to zero. On the other side, if the coefficient is large, it most likely will be different from zero in the ground-truth model, and hence its penalization should be decreased in order not to bias its value.

This approach is based on a constrained formulation that does not allow for training errors, since the resulting model will always satisfy $\vy = \mx \vw$. A possible implementation of the idea of \ref{EqProbRWB} without such a strong assumption is the following:
\begin{equation*}
 \label[problem]{EqProbRWD}
 \minpc{\vw \in \Rd}{\nt{\mx \vw - \vy}^2}{\sum_{i = 1}^{\ndim} \vti{i} \abs{\vwi{i}} \le 1} \eeq{,}
\end{equation*}
where the errors are minimized and the weighted $\lo$ regularizer is included as a constraint (equivalently, the regularizer could be also added to the objective function~\cite{Zou2006}). The iterative procedure to set the weights can still be the one explained above, where the weights at iteration $k$ are defined using \ref{EqRWLWeights}.

\subsection{Towards a Sparser SVM}

One important remark regarding the \rwla{} is that the re-weighting scheme breaks the equivalence with the SVM explained in \ref{SecEquiv}, i.e., one cannot simply apply the \rwla{} approach to solve the SVM problem in order to get more sparsity in the dual representation (i.e.\ fewer support vectors).
Instead, an analogous scheme will be directly included in the SVM formulation in the section below.

More specifically, and as shown in \ref{FigScheme}, the connection between Lasso and SVM suggests to apply a weighting scheme also for SVM. In order to set the weights, an iterative procedure (analogous to the \rwla{}) seems to be the natural step, although this would require to solve a complete SVM problem at each iteration. Finally, an online procedure to determine the weights, that are adapted directly at the optimization algorithm, will lead to a modification of the Frank--Wolfe algorithm.

{
\renewcommand{\plotline}[1]{{\color{box}$\blacksquare$}}
\begin{mfigure}{\label{FigScheme} Scheme of the relation between the proposed methods and the inspiring Lasso variants.}
 \legend{\showlegendcolour{box}~State-of-the-art Methods}{\showlegendcolour{box}~State-of-the-art Methods\seplegend\showlegendcolour{boxo}~Proposed Methods}
 \ifusetikz\scriptsize\tikzsetnextfilename{Scheme}{\input{./Scheme.tikz}}\else\includegraphics{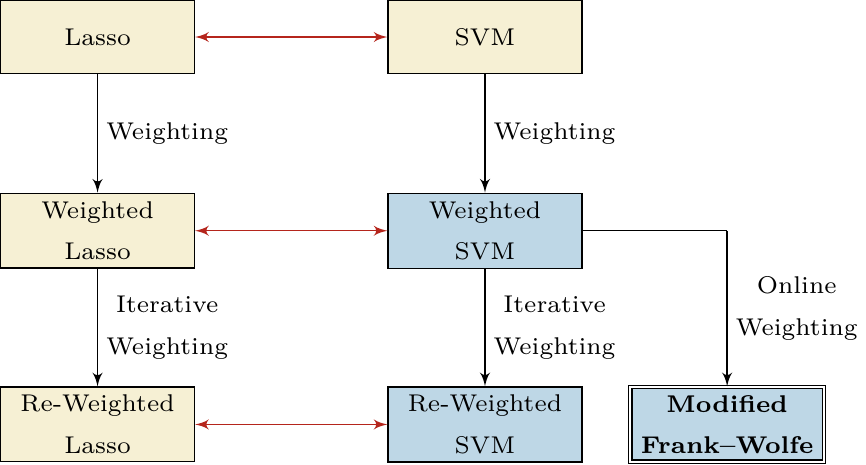}\fi
\end{mfigure}
}

It should be stated that a weighted SVM has been already proposed in~\cite{Lapin2014}, but that model differs from the approach described here. In particular, the weighing of \cite{Lapin2014} refers to the primal problem (through different regularization parameters associated to each pattern) whereas in this work the weighting refers directly to the dual problem. As explained in \ref{SSecWSVM}, both models are not equivalent.

\section{Weighted and Re-Weighted SVM}
\label{SecWSVM}

In this section the weighted SVM model is proposed. Furthermore, a re-weighting scheme to define iteratively the weights is sketched.

\subsection{Weighted SVM}
\label{SSecWSVM}

In order to transfer the weighting scheme of \rwla{} to an SVM framework, the most natural idea is to directly change the constraint of \ref{EqProbSVMD} to introduce the scaling factors $\vti{i}$. This results in the following Weighted-SVM (\wsvm{}) dual optimization problem:
\begin{equation}
 \label[problem]{EqProbWSVMD}
 \minpc{\va \in \Rp}{\va^\tr \mkh \va}{0 \le \vai{i} \sepcons \sum_{i=1}^{\npat} {\vti{i} \vai{i}} = 1} \eeq{,}
\end{equation}
for a fixed vector of weights $\vt$. This modification relates with the primal problem as stated in the proposition below.

\begin{mprop}
The \wsvm{} primal problem corresponding to \ref{EqProbWSVMD} is:
\begin{equation}
\label[problem]{EqProbWSVMP}
 \minpc{\vw, \rho, \xi}{\frac{1}{2} \nt{\vw}^2 - \rho + \frac{C}{2} \sum_{i = 1}^\npat {\xi_i^2}}{\vw^\tr \mzi{i} \ge \vti{i} \rho - \xi_i} \eeq{.}
\end{equation}
\end{mprop}
\begin{proof}
 The Lagrangian of \ref{EqProbWSVMP} is:
 \begin{equation*}
  \lagr\prn{\vw, \rho, \xi ; \va} = \frac{1}{2} \nt{\vw}^2 - \rho + \frac{C}{2} \sum_{i = 1}^\npat {\xi_i^2} + \sum_{i = 1}^\npat {\vai{i} \prn{- \vw^\tr \mzi{i} + \vti{i} \rho - \xi_i}} \eeq{,}
 \end{equation*}
 with derivatives with respect to the primal variables:
 \begin{align*}
  \deriv{\lagr}{\vw} &= \vw - \mz \va = 0 && \implies && \vw = \mz \va \eeq{;} \\
  \deriv{\lagr}{\rho} &= - 1 + \sum_{i = 1}^\npat {\vti{i} \vai{i}} = 0 && \implies && \sum_{i = 1}^\npat {\vti{i} \vai{i}} = 1 \eeq{;} \\
  \deriv{\lagr}{\xi} &= C \xi - \va = 0 && \implies && \xi = \frac{1}{C} \va \eeq{.}
 \end{align*}
 Substituting into the Lagrangian, the following objective function for the dual problem arises:
 \begin{equation*}
  \frac{1}{2} \nt{\mz \va}^2 - \rho + \frac{C}{2 C^2} \nt{\va}^2 - \nt{\mz \va}^2 + \rho \sum_{i = 1}^\npat {\vti{i} \vai{i}} - \frac{1}{C} \nt{\va}^2 = - \frac{1}{2} \nt{\mz \va}^2 - \frac{1}{2 C} \nt{\va}^2 \eeq{.}
 \end{equation*}
 Hence, the resulting dual problem is:
 \begin{equation*}
  \minpc{\va \in \Rp}{\nt{\mz \va}^2 + \frac{1}{C} \nt{\va}^2}{0 \le \vai{i} \sepcons \sum_{i = 1}^\npat {\vti{i} \vai{i}} = 1} \eeq{,}
 \end{equation*}
 which coincides with \ref{EqProbWSVMD}.
\end{proof}

Therefore, the effect of increasing the scaling factor $\vti{i}$ in the \wsvm{} dual formulation is equivalent to increasing the margin required for the $i$-th pattern in the primal formulation. Thus, intuitively an increase of $\vti{i}$ should facilitate the $i$-th pattern to become a support vector.
This influence is numerically illustrated in \ref{FigWSVMEvo}, where the value of one weight $\vti{i}$ is varied to analyse its influence over the corresponding multiplier $\vai{i}$ in a binary classification problem with $\npat = 100$ and $\ndim = 2$. The other weights are just fixed equal to one, but before solving the problem all the vector $\vt$ is normalized so that its maximum is still equal to one in order to preserve the scale. This experiment is done for three different values of $C$ (\num{e-3}, \num{1e0} and \num{e3}) and for the weights corresponding to the maximum, minimum and an intermediate value of the multiplier of the standard (unweighted) SVM. Clearly $\vti{i}$ and $\vai{i}$ present a proportional relationship, so the larger $\vti{i}$ is, the larger the obtained multiplier $\vai{i}$ becomes (until some point of saturation), confirming the initial intuition.

\begin{mfigure}{\label{FigWSVMEvo} Evolution of the SVM coefficient $\vai{}$ with respect to the weight $\vti{}$, for $C$ equal to \num{e-3} (first row), \num{e0} (second row) and \num{e3} (third row), and for the patterns corresponding to the maximum (first column), an intermediate (second column) and the minimum (third column) initial value of $\vai{}$.}
\colorlet{mycolor1}{graphic1}%
 \tikzwidth{0.3\textwidth}%
 \begin{footnotesize}
 \begin{tabular}{ccc}
  \ifusetikz\begin{minipage}{\figwidth}\vspace{0pt}\scriptsize\tikzsetnextfilename{ExampleEvolutionMax1e-03}{\tikzset{trim axis left,trim axis right}\input{./ExampleEvolutionMax1e-03.tikz}}\end{minipage}\else\includegraphics{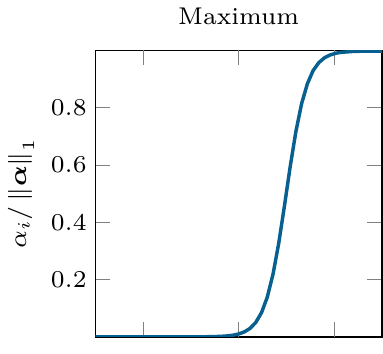}\fi &
  \ifusetikz\begin{minipage}{\figwidth}\vspace{0pt}\scriptsize\tikzsetnextfilename{ExampleEvolutionInt1e-03}{\tikzset{trim axis left,trim axis right}\input{./ExampleEvolutionInt1e-03.tikz}}\end{minipage}\else\includegraphics{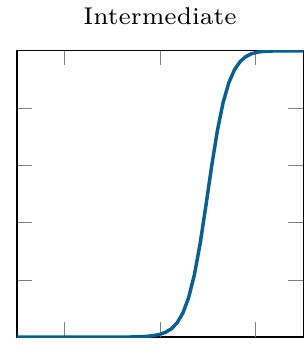}\fi &
  \ifusetikz\begin{minipage}{\figwidth}\vspace{0pt}\scriptsize\tikzsetnextfilename{ExampleEvolutionMin1e-03}{\tikzset{trim axis left,trim axis right}\input{./ExampleEvolutionMin1e-03.tikz}}\end{minipage}\else\includegraphics{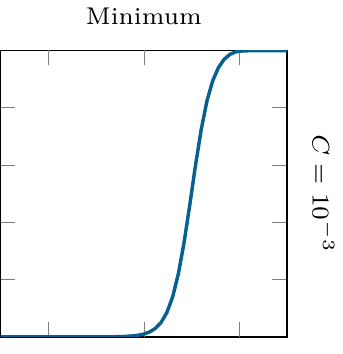}\fi \\
  \ifusetikz\begin{minipage}{\figwidth}\vspace{0pt}\scriptsize\tikzsetnextfilename{ExampleEvolutionMax1e+00}{\tikzset{trim axis left,trim axis right}\input{./ExampleEvolutionMax1e+00.tikz}}\end{minipage}\else\includegraphics{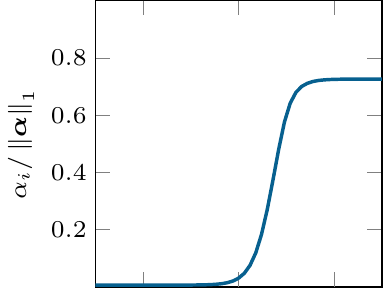}\fi &
  \ifusetikz\begin{minipage}{\figwidth}\vspace{0pt}\scriptsize\tikzsetnextfilename{ExampleEvolutionInt1e+00}{\tikzset{trim axis left,trim axis right}\input{./ExampleEvolutionInt1e+00.tikz}}\end{minipage}\else\includegraphics{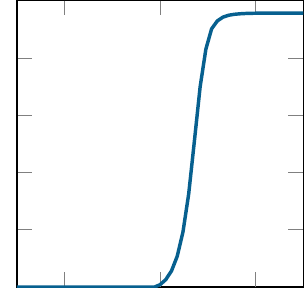}\fi &
  \ifusetikz\begin{minipage}{\figwidth}\vspace{0pt}\scriptsize\tikzsetnextfilename{ExampleEvolutionMin1e+00}{\tikzset{trim axis left,trim axis right}\input{./ExampleEvolutionMin1e+00.tikz}}\end{minipage}\else\includegraphics{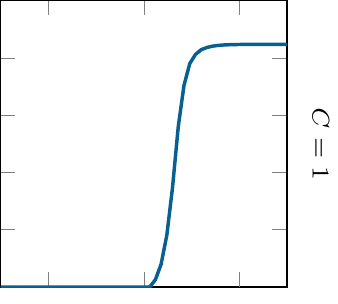}\fi\\
  \ifusetikz\begin{minipage}{\figwidth}\vspace{0pt}\scriptsize\tikzsetnextfilename{ExampleEvolutionMax1e+03}{\tikzset{trim axis left,trim axis right}\input{./ExampleEvolutionMax1e+03.tikz}}\end{minipage}\else\includegraphics{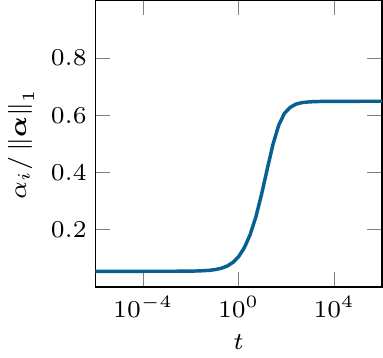}\fi &
  \ifusetikz\begin{minipage}{\figwidth}\vspace{0pt}\scriptsize\tikzsetnextfilename{ExampleEvolutionInt1e+03}{\tikzset{trim axis left,trim axis right}\input{./ExampleEvolutionInt1e+03.tikz}}\end{minipage}\else\includegraphics{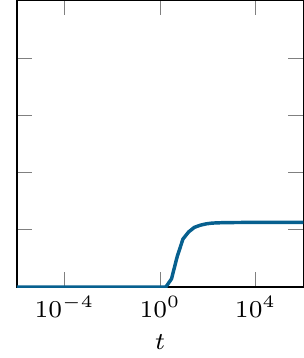}\fi &
  \ifusetikz\begin{minipage}{\figwidth}\vspace{0pt}\scriptsize\tikzsetnextfilename{ExampleEvolutionMin1e+03}{\tikzset{trim axis left,trim axis right}\input{./ExampleEvolutionMin1e+03.tikz}}\end{minipage}\else\includegraphics{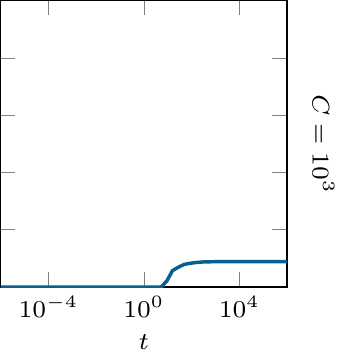}\fi
 \end{tabular}
 \end{footnotesize}
\end{mfigure}

As another illustration, \ref{FigWSVMFeasible} shows a small toy example of three patterns, which allows to represent the feasible set in two dimensions as the convex hull of the three vertices. The value of one weight $\vti{i}$ is changed in the set $\set{\num{e-2}, \num{e-1}, \num{e0}, \num{e1}, \num{e2}}$, whereas the other two weights are kept fixed to $1$. As before, increasing the weight pushes the solution towards the corresponding pattern.
Moreover, the last row in \ref{FigWSVMFeasible} shows the same example but with a three dimensional representation, so that it is more clear the effect of decreasing $\vti{1}$ in the feasible set, basically lengthening the triangle and increasing its angle with respect to the horizontal plane, until the point where the triangle becomes an unbounded rectangle ($\vti{1} = 0$) completely vertical. Taking into consideration that the solution of the unconstrained problem (for $C \neq \infty$) is the origin, decreasing $\vti{1}$ is moving away the first vertex from the unconstrained solution, thus making less likely to assign a non-zero coefficient to that point unless it really decreases the objective function.

{
 \pgfplotscreateplotcyclelist{mycolorlist}{%
graphic5,mark size=1pt,mark=*\\%
 }
 \pgfplotsset{xtick=\empty,ytick=\empty,ztick=\empty,axis on top=true,enlargelimits=false,every axis/.append style={axis lines=center,ylabel style={anchor=north east,inner sep=1pt,outer sep=1pt},zlabel style={anchor=north west,inner sep=1pt,outer sep=1pt},xlabel style={anchor=south,inner sep=1pt,outer sep=1pt}},y={(-0.4\textwidth,-0.230940108\textwidth)},z={(0.4\textwidth,-0.230940108\textwidth)},x={(0,0.461880215\textwidth)},cycle list name=mycolorlist}
 \renewcommand{\plotline}[1]{%
  \tikzset{external/export next=false}%
  \begin{tikzpicture}[]
    \begin{axis}[hide axis, scale only axis,width=5pt, height=5pt, xmin=0, xmax=1, ymin=0, ymax=2, cycle list name=mycolorlist, cycle list shift=#1,x={},y={},z={}]
      \addplot
      coordinates {
      (0.5,1)
      };
    \end{axis}
  \end{tikzpicture}}
\renewcommand{\showlegend}[1]{[\,\raisebox{\height}{\plotline{#1}}\,]}
\begin{mfigurec}{\label{FigWSVMFeasible} Example of the feasible region and the solution for a problem with three patterns, for different values of the weighting vector $\vt$. For each plot, the value of $\vt$ is shown above in boldface. The three rows correspond to changes in $\vti{1}$, $\vti{2}$ and $\vti{3}$ respectively, and the weighted probability simplex is represented as the convex hull of the three vertices. The fourth row corresponds again to changes in $\vti{1}$ but with a $3$-dimensional representation keeping the same aspect ratio for all the axis, and also including the limit case $\vti{1} = 0$ where $\vai{1}$ is not upper bounded. The solution of the constrained optimization problem is shown with a red dot~\showlegend{0}.}
 \tikzwidth{0.125\textwidth}%
\scriptsize
\ifusetikz
 \begin{tabular}{*5{@{\hspace{10pt}}c@{\hspace{10pt}}}}
\else
 \begin{tabular}{*5c}
\fi
  \toprule
  \tabformathrow{$\prns{1, 1, \num{e-2}}$} & \tabformathrow{$\prns{1, 1, \num{e-1}}$} & \tabformathrow{$\prns{1, 1, \num{e0}}$} & \tabformathrow{$\prns{1, 1, \num{e1}}$} & \tabformathrow{$\prns{1, 1, \num{e2}}$} \\\midrule
  \ifusetikz\begin{minipage}{\figwidth}\vspace{0pt}\scriptsize\tikzsetnextfilename{FeasibleSet-1-1}{\tikzset{trim axis left,trim axis right}\input{./FeasibleSet-1-1.tikz}}\end{minipage}\else\includegraphics{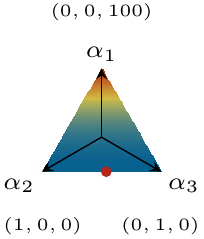}\fi & \ifusetikz\begin{minipage}{\figwidth}\vspace{0pt}\scriptsize\tikzsetnextfilename{FeasibleSet-2-1}{\tikzset{trim axis left,trim axis right}\input{./FeasibleSet-2-1.tikz}}\end{minipage}\else\includegraphics{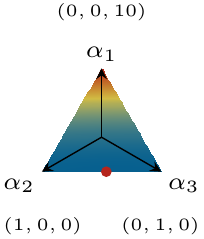}\fi & \ifusetikz\begin{minipage}{\figwidth}\vspace{0pt}\scriptsize\tikzsetnextfilename{FeasibleSet-3-1}{\tikzset{trim axis left,trim axis right}\input{./FeasibleSet-3-1.tikz}}\end{minipage}\else\includegraphics{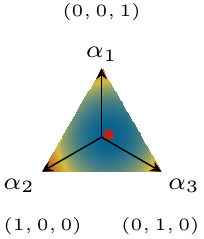}\fi & \ifusetikz\begin{minipage}{\figwidth}\vspace{0pt}\scriptsize\tikzsetnextfilename{FeasibleSet-4-1}{\tikzset{trim axis left,trim axis right}\input{./FeasibleSet-4-1.tikz}}\end{minipage}\else\includegraphics{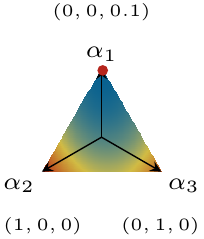}\fi & \ifusetikz\begin{minipage}{\figwidth}\vspace{0pt}\scriptsize\tikzsetnextfilename{FeasibleSet-5-1}{\tikzset{trim axis left,trim axis right}\input{./FeasibleSet-5-1.tikz}}\end{minipage}\else\includegraphics{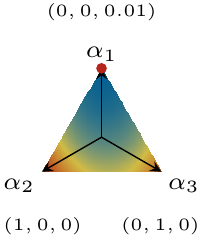}\fi \\[5pt]\toprule
  \tabformathrow{$\prns{1, \num{e-2}, 1}$} & \tabformathrow{$\prns{1, \num{e-1}, 1}$} & \tabformathrow{$\prns{1, \num{e0}, 1}$} & \tabformathrow{$\prns{1, \num{e1}, 1}$} & \tabformathrow{$\prns{1, \num{e2}, 1}$} \\\midrule
  \ifusetikz\begin{minipage}{\figwidth}\vspace{0pt}\scriptsize\tikzsetnextfilename{FeasibleSet-1-2}{\tikzset{trim axis left,trim axis right}\input{./FeasibleSet-1-2.tikz}}\end{minipage}\else\includegraphics{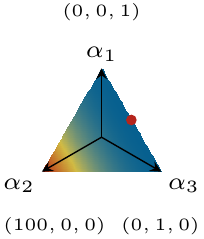}\fi & \ifusetikz\begin{minipage}{\figwidth}\vspace{0pt}\scriptsize\tikzsetnextfilename{FeasibleSet-2-2}{\tikzset{trim axis left,trim axis right}\input{./FeasibleSet-2-2.tikz}}\end{minipage}\else\includegraphics{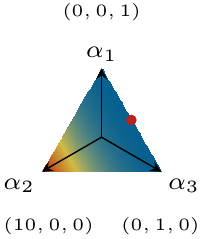}\fi & \ifusetikz\begin{minipage}{\figwidth}\vspace{0pt}\scriptsize\tikzsetnextfilename{FeasibleSet-3-2}{\tikzset{trim axis left,trim axis right}\input{./FeasibleSet-3-2.tikz}}\end{minipage}\else\includegraphics{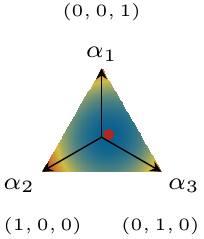}\fi & \ifusetikz\begin{minipage}{\figwidth}\vspace{0pt}\scriptsize\tikzsetnextfilename{FeasibleSet-4-2}{\tikzset{trim axis left,trim axis right}\input{./FeasibleSet-4-2.tikz}}\end{minipage}\else\includegraphics{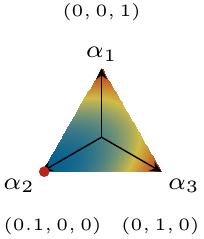}\fi & \ifusetikz\begin{minipage}{\figwidth}\vspace{0pt}\scriptsize\tikzsetnextfilename{FeasibleSet-5-2}{\tikzset{trim axis left,trim axis right}\input{./FeasibleSet-5-2.tikz}}\end{minipage}\else\includegraphics{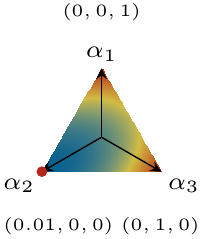}\fi \\\toprule
  \tabformathrow{$\prns{\num{e-2}, 1, 1}$} & \tabformathrow{$\prns{\num{e-1}, 1, 1}$} & \tabformathrow{$\prns{\num{e0}, 1, 1}$} & \tabformathrow{$\prns{\num{e1}, 1, 1}$} & \tabformathrow{$\prns{\num{e2}, 1, 1}$} \\\midrule
  \ifusetikz\begin{minipage}{\figwidth}\vspace{0pt}\scriptsize\tikzsetnextfilename{FeasibleSet-1-3}{\tikzset{trim axis left,trim axis right}\input{./FeasibleSet-1-3.tikz}}\end{minipage}\else\includegraphics{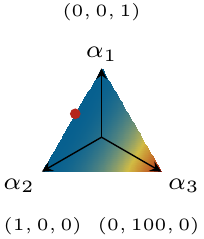}\fi & \ifusetikz\begin{minipage}{\figwidth}\vspace{0pt}\scriptsize\tikzsetnextfilename{FeasibleSet-2-3}{\tikzset{trim axis left,trim axis right}\input{./FeasibleSet-2-3.tikz}}\end{minipage}\else\includegraphics{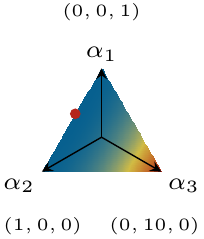}\fi & \ifusetikz\begin{minipage}{\figwidth}\vspace{0pt}\scriptsize\tikzsetnextfilename{FeasibleSet-3-3}{\tikzset{trim axis left,trim axis right}\input{./FeasibleSet-3-3.tikz}}\end{minipage}\else\includegraphics{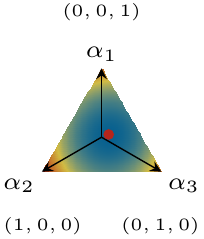}\fi & \ifusetikz\begin{minipage}{\figwidth}\vspace{0pt}\scriptsize\tikzsetnextfilename{FeasibleSet-4-3}{\tikzset{trim axis left,trim axis right}\input{./FeasibleSet-4-3.tikz}}\end{minipage}\else\includegraphics{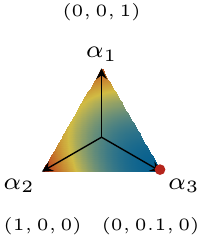}\fi & \ifusetikz\begin{minipage}{\figwidth}\vspace{0pt}\scriptsize\tikzsetnextfilename{FeasibleSet-5-3}{\tikzset{trim axis left,trim axis right}\input{./FeasibleSet-5-3.tikz}}\end{minipage}\else\includegraphics{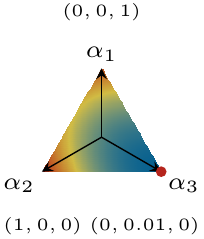}\fi \\\toprule
  \tabformathrow{$\prns{\num{4}, 1, 1}$} & \tabformathrow{$\prns{\num{2}, 1, 1}$} & \tabformathrow{$\prns{\num{1}, 1, 1}$} & \tabformathrow{$\prns{\num{0.5}, 1, 1}$} & \tabformathrow{$\prns{0, 1, 1}$} \\\midrule
  \ifusetikz\begin{minipage}{\figwidth}\vspace{0pt}\scriptsize\tikzsetnextfilename{FeasibleSetUnboundedEvo-1}{\tikzset{trim axis left,trim axis right}\input{./FeasibleSetUnboundedEvo-1.tikz}}\end{minipage}\else\includegraphics{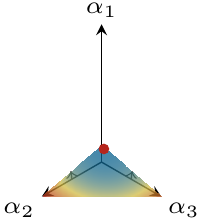}\fi & \ifusetikz\begin{minipage}{\figwidth}\vspace{0pt}\scriptsize\tikzsetnextfilename{FeasibleSetUnboundedEvo-2}{\tikzset{trim axis left,trim axis right}\input{./FeasibleSetUnboundedEvo-2.tikz}}\end{minipage}\else\includegraphics{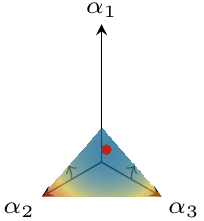}\fi & \ifusetikz\begin{minipage}{\figwidth}\vspace{0pt}\scriptsize\tikzsetnextfilename{FeasibleSetUnboundedEvo-3}{\tikzset{trim axis left,trim axis right}\input{./FeasibleSetUnboundedEvo-3.tikz}}\end{minipage}\else\includegraphics{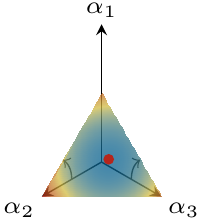}\fi & \ifusetikz\begin{minipage}{\figwidth}\vspace{0pt}\scriptsize\tikzsetnextfilename{FeasibleSetUnboundedEvo-4}{\tikzset{trim axis left,trim axis right}\input{./FeasibleSetUnboundedEvo-4.tikz}}\end{minipage}\else\includegraphics{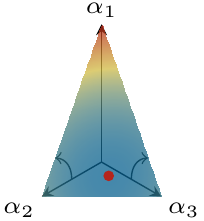}\fi & \ifusetikz\begin{minipage}{\figwidth}\vspace{0pt}\scriptsize\tikzsetnextfilename{FeasibleSetUnboundedEvo-5}{\tikzset{trim axis left,trim axis right}\input{./FeasibleSetUnboundedEvo-5.tikz}}\end{minipage}\else\includegraphics{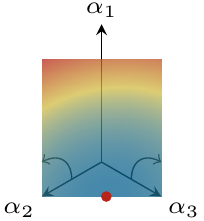}\fi \\\bottomrule
 \end{tabular}
\end{mfigurec}
}

It is mandatory to state the differences between the \wsvm{} proposed here and the previous model proposed in~\cite{Lapin2014}. First of all, the formulations over which both approach are based are different.
But, even if the same starting SVM model were used, both weighting schemes are essentially different:
\begin{itemize}
 \item Lapin \emph{et al.} propose a modification of the primal SVM formulation so that the cost associated is different for each pattern. This means that the loss associated to that pattern is multiplied by a constant.
 \item The \wsvm{} proposed here introduces the weights directly into the dual problem, what results into a modification of the margin required for each pattern in the primal problem. Considering again the loss associated to each pattern, the ``insensitivity'' zone (the region of predictions that are associated to a zero loss) is widened or narrowed according to a constant.
\end{itemize}
Hence, both approaches are fundamentally different, and the effects that they produce are not equivalent.

\subsection{Re-Weighted SVM}

Once \ref{EqProbWSVMD} has been defined, and provided that the scaling factors seem to influence the sparsity of the solution (as illustrated in \ref{FigWSVMEvo,FigWSVMFeasible}), a procedure to set the weighting vector $\vt$ is needed.

In parallelism with the original \rwla{}, but considering that in this case the relation between the weight $\vti{i}$ and the corresponding optimal multiplier $\vai{i}$ is directly proportional, the following iterative approach, namely Re-Weighted SVM (\rwsvm{}), arises naturally:
\begin{enumerate}
 \item At iteration $k$, the following \wsvm{} problem is solved:
  \begin{equation}
  \label[problem]{EqProbRWSVMD}
  \vaok{k} = \argminpcd{\va \in \Rp}{\va^\tr \mkh \va}{0 \le \vai{i} \le 1 \sepcons \sum_{i=1}^{\npat} {\vtik{i}{k} \vai{i}} = 1 }
 \end{equation}
 \item The weighting vector for the next iteration, $\vtk{k}$, is updated as:
 \begin{equation*}
  \vtik{i}{k + 1} = \fmon\prn{\vaiok{i}{k}} \eeq{,}
 \end{equation*}
 where $\fmon: \R \to \R$ is some monotone function.
\end{enumerate}

This approach has two main drawbacks.
The first one is how to select the function $\fmon$. This also implies selecting some minimum and maximum values to which the weights $\vtik{i}{k}$ should saturate, so it is not a trivial task, and it can greatly influence the behaviour of the model.
The second drawback is that this approach requires to solve \ref{EqProbRWSVMD} at each iteration, which means training completely a \wsvm{} model (with a complexity that should not differ from that of training a standard SVM) on each iteration, and hence the overall computational cost can be much larger. Although this is in fact an affordable drawback if the objective is solely to approach the \lz{} norm as it was the case in the original paper of \rwla{}~\cite{Candes2008}, in the case of the SVM the aim is to get sparser models in order to reduce their complexity and to improve the performance specially in large datasets, and hence it does not make sense to need for this several iterations.

As a workaround, the next section proposes an online modification of the weights that leads to a simple modification of the training algorithm for SVMs.

\section{Modified Frank--Wolfe Algorithm}
\label{SecMFW}

This section proposes a training algorithm to get sparser SVMs, which is based on an online modification of the weighting vector $\vt$ of a \wsvm{} model.
In particular, the basis of this proposal is the Frank--Wolfe optimization algorithm.

\subsection{Frank--Wolfe Algorithm}

The Frank--Wolfe algorhtm (\fw{};~\cite{Frank1956}) is a first order optimization method for constrained convex optimization. There are several versions of this algorithm, in particular the basis of this work is the Pairwise Frank--Wolfe~\cite{Jaggi2013,Lacoste-Julien2015}.
Roughly speaking, it is based on using at each iteration a linear approximation of the objective function to select one of the vertices as the target towards which the current estimate of the solution will move (the \emph{forward} node), and another vertex as that from which the solution will move away (the \emph{away} node), and then updating the solution in the direction that goes from the \emph{away} node to the \emph{forward} one using the optimal step length. At the end, the linear approximation boils down to selecting the node corresponding to the smallest partial derivative as the \emph{forward} node, and that with the largest derivative as the \emph{away} node.

This general algorithm can be used in many different contexts, and in particular it has been succesfully applied to the training of SVMs~\cite{Gaertner2009,Ouyang2010,Frandi2013}. Specifically, for the case of \ref{EqProbSVMD}, the following definitions and results are employed.

Let $\fobj$ denote the (scaled) objective function of \ref{EqProbSVMD} (or \ref{EqProbWSVMD}), with gradient and partial derivatives:
\begin{align}
 \fobj\prn{\va} &= \frac{1}{2} \va^\tr \mkh \va \eeq{;} \label{Eqfobj} \\
 \nabla \fobj\prn{\va} &= \mkh \va \eeq{;} \nonumber \\
 \frac{\partial \fobj}{\partial \vai{i}}\prn{\va} &= \mkhi{i}^\tr \va \eeq{,} \label{Eqgobj}
\end{align}
where $\mkhi{i}^\tr$ is the $i$-th row of $\mkh$.
Let $\vdir$ denote the direction in which the current solution will be updated. The optimal step-size can be computed by solving the problem:
\begin{equation}
 \label[problem]{EqProbStep}
 \minp{\stpg}{\fobj\prn{\va + \stpg \vdir}} \eeq{,}
\end{equation}
and truncating the optimal step, if needed, in order to remain in the convex hull of the nodes, i.e., to satisfy the constraints of \ref{EqProbSVMD} (or, equivalently, of \ref{EqProbWSVMD}).
Straightforwardly, \ref{EqProbStep} can be solved simply taking the derivative with respect to $\stpg$ and making it equal to zero:
\begin{align*}
 \frac{\partial \fobj}{\partial \stpg}\prn{\va + \stpg \vdir} &= \vdir^\tr \mkh \prn{\va + \stpg \vdir} = 0 \\
 \implies \stpg &= - \frac{\vdir^\tr \mkh \va}{\vdir^\tr \mkh \vdir} \eeq{.}
\end{align*}
It should be noticed that $\mkh \va$ is the gradient of $\fobj$ at the point $\va$, and thus there is no need to compute it again (indeed, the gradient times the direction is minus the \fw{} gap, that can be used as a convergence indicator). Moreover, if the direction $\vdir = \sum_{i \in \upd}{\vdiri{i} \bas{i}}$ is sparse, then $\mkh \vdir = \sum_{i \in \upd} \vdiri{i} \mkhi{i}$ only requires to compute the columns of the kernel matrix corresponding to the set of updated variables $\upd$. In particular, in the Pairwise \fw{} only the columns of the \emph{forward} and \emph{away} nodes are used to determine $\stpg$ and to keep the gradient updated.

The whole procedure for applying \fw{} to the SVM training is summarized in \ref{AlgFW}.

\begin{malgorithm}{\label{AlgFW} Pairwise Frank--Wolfe algorithm for SVM.}
\begin{algorithmic}[1]
 \Procedure{TrainSVM}{$\mkh, \epsilon$} \Comment{\textbullet~Kernel $\mkh \in \Rpp$. \\\textbullet~Precision $\epsilon \in \R$.}
  \algorithmiccommentb{Initialization.}
  \Take $i_0 \in \set{1, \dotsc, \npat}$ \Comment{Initial vertex.}
  \State $\va \gets \bas{i_0}$ \Comment{Initial point.}
  \State $\vg \gets \vkhi{i_0}$ \Comment{Initial gradient.}
  \Repeat \Comment{Main Loop.}
   \algorithmiccommentb{Update of Coefficients.}
   \State $s \gets \argmin_{i}{\vgi{i}}$ \Comment{Select forward node.}
   \State $v \gets \argmax_{i}{\vgi{i}}$ \Comment{Select away node.}
   \State $\vdir \gets \bas{s} - \bas{v}$ \Comment{Build update direction.}
   \State $\gap \gets - \vg \cdot \vdir$ \Comment{\fw{} gap.}
   \State $\stpg \gets \minme{\maxme{\gap/\prns{\vdir^\tr \mkh \vdir}, 0}, \vai{v}}$ \Comment{Compute step length.}
   \State $\va \gets \va + \stpg \vdir$ \Comment{Point update.}
   \State $\vg \gets \vg + \stpg \mkhi{s} - \stpg \mkhi{v}$ \Comment{Gradient update.}
  \Until{$\gap \le \epsilon$} \Comment{Stopping criterion.}
 \EndProcedure 
\end{algorithmic}
\end{malgorithm}

\subsection{Modified Frank--Wolfe Algorithm}
\label{SWSVM}

The idea of the proposed Modified Frank--Wolfe (\mfw{}) is to modify the weights $\vti{i}$, i.e., the margin required for each training pattern, directly on each inner iteration of the algorithm, hence with an overall cost similar to that of the original \fw{}.
In particular, and since according to \ref{FigWSVMEvo,FigWSVMFeasible} the relation between each weight and the resulting coefficient seems to be directly proportional, an incremental procedure with binary weights is defined, leading to a new training algorithm for SVM. Specifically, the training vectors will be divided into two groups, the working vectors, with a weight $\vti{i} = 1$, and the idle vectors, with a weight $\vti{i} = 0$.
The proposed \mfw{} will start with only one initial working vector, and at each iteration, the idle vector with the smaller negative gradient (if there is any) will be add to the working set. After that, the coefficients of the working vectors will be updated by using a standard \fw{} pair-wise step.

The intuition behind this algorithm is the following.
The standard \fw{} algorithm applied to the SVM training will activate (make non-zero) the coefficient of a certain vector if its partial derivative is better (smaller) than that of the already active coefficients, i.e., if that vector is ``less bad'' than the others.
On the other side, the \mfw{} will only add a coefficient to the working set if its partial derivative is negative (hence, that coefficient would also be activated without the simplex constraint), i.e., the vector has to be somehow ``good'' by itself.

In what follows, the \mfw{} algorithm is described in more detail.

\subsubsection{Preliminaries}

The set of working vectors is denoted by $\act = \setst{i}{1 \le i \le \npat, \vti{i} = 1}$, and that of idle vectors as $\ina = \setst{i}{1 \le i \le \npat, \vti{i} = 0}$.
The dual problem becomes:
\begin{equation*}
\label[problem]{EqProbWSVMB}
 \minpc{\va \in \Rp}{\va^\tr \mkh \va}{0 \le \vai{i} \sepcons \sum_{i \in \act} {\vai{i}} = 1} \eeq{.}
\end{equation*}
Thus, the coefficients for the points in $\act$ have to belong to the probability simplex of dimension $\abs{\act}$, whereas the coefficients for $\ina$ only have a non-negative constraint.

\subsubsection{Algorithm}

The proposed \mfw{} algorithm to train an SVM is summarized in \ref{AlgMFW}.
This algorithm is very similar to \ref{AlgFW}, except for the initialization and control of the working set in \ref{LNactini,LNactBeg,LNacta,LNactb,LNactc,LNactd,LNacte,LNactf,LNactEnd}, the search for the \emph{forward} and \emph{away} nodes of \ref{LNaway,LNforward} (which is done only over the working set) and the stopping criterion of \ref{LNstop} (which requires both that the dual gap is small enough and that no new vertices have been activated).

\begin{malgorithm}{\label{AlgMFW} Modified Frank--Wolfe algorithm for SVM.}
\begin{algorithmic}[1]
 \Procedure{Train\isvmname{}}{$\mkh, \epsilon$} \Comment{\textbullet~Kernel $\mkh \in \Rpp$. \\\textbullet~Precision $\epsilon \in \R$.}
  \algorithmiccommentb{Initialization.}
  \Take $i_0 \in \set{1, \dotsc, \npat}$ \Comment{Initial vertex.}
  \State $\act \gets \set{i_0}$ \Comment{Initial working set.}\label{LNactini}
  \State $\va \gets \bas{i_0}$ \Comment{Initial point.}
  \State $\vg \gets \vkhi{i_0}$ \Comment{Initial gradient.}
  \Repeat \Comment{Main Loop.}
   \algorithmiccommentb{Activation of Coefficients.}
   \If {$\abs{\act} < \npat$}\label{LNactBeg}
    \State $\fchn \gets \text{false}$ \Comment{Flag for changes.}\label{LNacta}
    \State $u \gets \argmin_{i \in \ina}{\vgi{i}}$ \Comment{Select node.}\label{LNactb}
    \If {$\vgi{u} < 0$}\label{LNactc}
     \State $\act \gets \act \cup \set{u}$ \Comment{Activate node.}\label{LNactd}
     \State $\fchn \gets \text{true}$ \Comment{Mark change.}\label{LNacte}
    \EndIf\label{LNactf}
   \EndIf\label{LNactEnd}
   \algorithmiccommentb{Update of Working Coefficients.}
   \State $s \gets \argmin_{i \in \act}{\vgi{i}}$ \Comment{Select forward node.}\label{LNforward}
   \State $v \gets \argmax_{i \in \act}{\vgi{i}}$ \Comment{Select away node.}\label{LNaway}
   \State $\vdir \gets \bas{s} - \bas{v}$ \Comment{Build update direction.}
   \State $\gap \gets - \vg \cdot \vdir$ \Comment{\fw{} gap.}
   \State $\stpg \gets \minme{\maxme{\gap/\prns{\vdir^\tr \mkh \vdir}, 0}, \vai{v}}$ \Comment{Compute step length.}
   \State $\va \gets \va + \stpg \vdir$ \Comment{Point update.}
   \State $\vg \gets \vg + \stpg \mkhi{s} - \stpg \mkhi{v}$ \Comment{Gradient update.}
  \Until{$\gap \le \epsilon$ and not $\fchn$} \Comment{Stopping criterion.}\label{LNstop}
 \EndProcedure 
\end{algorithmic}
\end{malgorithm}

\subsubsection{Convergence}

Regarding the convergence of the \mfw{} algorithm, the following theorem states that this algorithm will provide a model that is equivalent to a standard SVM model trained only over a subsample\footnote{Due to its sparse nature, an SVM is expressed only in terms of the support vectors. Nevertheless, the proposed \mfw{} provides an SVM trained over a subsample of the training set, although not all the vectors of this subsample have to become support vectors.} of the training patterns.
\begin{mtheor}
\label{TheoConvergence}
 \Ref{AlgMFW} converges to a certain vector $\opt{\va} \in \Rp$. In particular:
 \begin{enumerate}[(i)]
  \item The working set converges to a set $\acto$.
  \item The components of $\opt{\va}$ corresponding to $\acto$ conform the solution of the standard SVM \ref{EqProbSVMD} posed over the subset $\acto$ of the set of training patterns. The remaining components $\opt{\vai{i}}$, for $i \notin \acto$, are equal to zero.
 \end{enumerate}
\end{mtheor}
\begin{proof}
 {\quad}
 \begin{enumerate}[(i)]
  \item Let $\actk{k}$ denote the working set at iteration $k$. At iteration $k + 1$, the set $\actk{k + 1}$ will be either equal to $\actk{k}$ or equal to $\actk{k} \cup \set{u}$ for some $u \notin \actk{k}$. Hence, $\actk{k} \subseteq \actk{k + 1}$ for all $k$. Moreover, $\actk{k}$ is always a subset of the whole set of training vectors $\tset = \set{1, \cdots, \npat}$, i.e. $\actk{k} \subseteq \tset$ for all $k$. Since $\set{\actk{k}}$ is a monotone nondecreasing sequence of subsets of a finite set $\tset$, then $\actk{k} \uparrow \acto \subseteq \tset$, as proved next. Let $\set{k_1, \dotsc, k_{\npatp}}$ be those iterations in which the working set grows, $\actk{k_i} \subset \actk{k_i + 1}$. Obviously, the number of such $\npatp$ iterations is finite with $\npatp \le \npat$ since no more than $\npat$ elements can be added to the working set. Therefore, $\forall k \ge k_{\npatp}$, $\actk{k} = \actk{k_\npatp} = \acto \subseteq \tset$.
 \item Provided that $\actk{k} \subseteq \acto$ for all $k$, then for $i \notin \acto$ the corresponding coefficients will never be updated (they cannot be selected in \ref{LNforward,LNaway} of \ref{AlgMFW}), so they would conserve they initial value, i.e. $\vaik{i}{k} = 0$ for all $i \notin \acto$ and for all $k$.

 With respect to the convergence of $\vaik{i}{k}$ for $i \in \acto$, it suffices to consider the iterations after the convergence of the working set, $k \ge k_{\npatp}$. Let $\vak{k}_{\acto} \in \R^{\npatp}$ be the vector composed by the coefficients of the working patterns. Using \ref{Eqfobj} and since the coefficients of idle vectors are equal to zero (proved above):
 \begin{equation*}
  \fobj\prn{\vak{k}} = \sum_{i = 1}^\npat \sum_{j = 1}^\npat \vkhi{ij} \vaik{i}{k} \vaik{j}{k} = \sum_{i \in \acto} \sum_{j \in \acto} \vkhi{ij} \vaik{i}{k} \vaik{j}{k} = \fobj_{\acto}\prn{\vak{k}_{\acto}} \eeq{,}
 \end{equation*}
 where $\fobj_{\acto}$ denotes the objective function of \ref{EqProbSVMD} posed only over the subset $\acto$ of the original training set. A similar result can be obtained for the components of the gradient using \ref{Eqgobj}:
 \begin{equation*}
  \frac{\partial \fobj}{\partial \vai{i}}\prn{\vak{k}} = \sum_{j = 1}^\npat \vkhi{ij} \vaik{j}{k} = \sum_{j \in \acto}^\npat \vkhi{ij} \vaik{j}{k} = \frac{\partial \fobj_{\acto}}{\partial \vai{i}}\prn{\vak{k}_{\acto}} \eeq{.}
 \end{equation*}
 Therefore, once the working set has converged both the objective function and the partial derivatives of the working set computed in \ref{AlgMFW} are equal to those computed in \ref{AlgFW} when this algorithm is applied only over the vectors of the working set. Hence, in the remaining iterations \mfw{} reduces to the standard \fw{} algorithm but considering only the vertices in $\acto$, which converges to the solution of \ref{EqProbSVMD} over the subset $\acto$~\cite{Lacoste-Julien2015}.
 \end{enumerate}
\end{proof}

It is worth mentioning that, although the proposed \mfw{} algorithm converges to an SVM model trained over a subsample $\acto$ of the training data, this subsample will (as shown in \ref{SecExp}) depend on the initial point of the algorithm.

\section{Experiments}
\label{SecExp}

In this section the proposed \mfw{} algorithm will be compared with the standard \fw{} algorithm over several classification tasks.
In particular, the binary datasets that will be used for the experiments are described in \ref{TabDatasets}, which includes the size of the training and test sets, the number of dimensions and the percentage of the majority class (as a baseline accuracy). All of them belong to the LibSVM repository~\cite{Chang2011} except for \dmgam{} and \dmini{}, which belong to the UCI repository~\cite{Lichman2015}.

\begin{mtable}{\label{TabDatasets} Description of the datasets.}
 \begin{tabular}{c
    S[table-format=5]
    S[table-format=5]
    S[table-format=3]
    S[table-format=2.1]
    }
  \toprule
  \tabformathrow{Dataset} & \tabformathrow{Tr. Size} & \tabformathrow{Te. Size} & \tabformathrow{Dim.} & \tabformathrow{Maj. Class (\%)} \\
  \midrule
   \dijcn{} & 49990 & 91701 & 22 & 90.4 \\
\dmgam{} & 13020 & 6000 & 10 & 64.8 \\
\midrule
\daust{} & 621 & 69 & 14 & 55.5 \\
\dbrea{} & 615 & 68 & 10 & 65.0 \\
\ddiab{} & 692 & 76 & 8 & 65.1 \\
\dgerm{} & 900 & 100 & 24 & 70.0 \\
\dhear{} & 243 & 27 & 13 & 55.6 \\
\diono{} & 316 & 35 & 34 & 64.1 \\
\diris{} & 135 & 15 & 4 & 66.7 \\
\dmush{} & 7312 & 812 & 112 & 51.8 \\
\dsona{} & 188 & 20 & 60 & 53.4 \\
\midrule
\dmini{} & 100000 & 29596 & 50 & 71.8 \\

  \bottomrule
 \end{tabular}
\end{mtable}

\subsection{Preliminary Experiments}

The first experiments will be focused on the first two datasets of \ref{TabDatasets}, namely \dijcn{} and \dmgam{}, which are the largest ones except for \dmini{}.

\subsubsection{Set-Up}

The standard SVM model trained using \fw{} (\ssvm{}) and the model resulting from the proposed \mfw{} algorithm (denoted by \isvm{}, which as shown in \ref{TheoConvergence} is just an SVM trained over a subsample $\acto$ of the original training set) will be compared in terms of their accuracies, the number of support vectors and the number of iterations needed to achieve the convergence during the training algorithm. Two different kernels will be used, the linear and the RBF (or Gaussian) ones.
With respect to the hyper-parameters of the models, the value of both $C$ and the bandwidth $\sigma$ (in the case of the RBF kernel) will be obtained through \num{10}-fold Cross Validation (CV) for \dmgam{}, whereas for the largest dataset \dijcn{} only $C$ will be tuned, and $\sigma$ will be fixed as $\sigma = 1$ in the RBF kernel (this value is similar to the one used for the winner of the IJCNN competition~\cite{Chang2001}).
Once the hyper-parameters are tuned, both models will be used to predict over the test sets.
The stopping criterion used is $\epsilon = \num{e-5}$.

\subsubsection{Results}

The test results are summarized in \ref{TabResultsIncLarge}.
Looking first at the accuracies, both models \ssvm{} and \isvm{} are practically equivalent in three of the four experiments, where the differences are insignificant, whereas for \dijcn{} with the linear kernel the accuracy is higher in the case of \isvm{}.
Regarding the number of support vectors, \isvm{} gets sparser models for \dijcn{} with linear kernel and \dmgam{} with RBF kernel, whereas for the other two experiments both models get a comparable sparsity.
Finally, and with respect to the convergence of the training algorithms, \isvm{} shows an advantage when dealing with linear kernels, whereas for the RBF ones both approaches are practically equivalent.

\begin{mtable}{\label{TabResultsIncLarge} Test results for the larger datasets.}
\sisetup{output-exponent-marker=\textsc{e},exponent-product={},retain-explicit-plus}
\renewcommand{\widthbox}{\widthof{$\num{9.99e+9}$}}
 \begin{tabular}{l@{}c*6{c}}
  \toprule
  \multirow{2}{*}{\tabformathrow{Data}} & \multirow{2}{*}{\tabformathrow{K.}} & \multicolumn{2}{c}{\tabformathrow{Accuracy ($\%$)}} & \multicolumn{2}{c}{\tabformathrow{Number SVs}} & \multicolumn{2}{c}{\tabformathrow{Number Iters.}} \\
  \cmidrule(lr){3-4} \cmidrule(lr){5-6} \cmidrule(lr){7-8}
   & & \tabformathrow{\ssvm{}} & \tabformathrow{\isvm{}}  & \tabformathrow{\ssvm{}} & \tabformathrow{\isvm{}}  & \tabformathrow{\ssvm{}} & \tabformathrow{\isvm{}} \\
  \midrule 
  \datasettitle{\dijcn{}}
 & lin & $\num{92.17}$ & $\num{93.20}$ & $\num{2.01e+04}$ & $\num{8.00e+03}$ & $\num{5.39e+04}$ & $\num{1.75e+04}$ \\
 & rbf & $\num{98.83}$ & $\num{98.81}$ & $\num{4.99e+03}$ & $\num{4.98e+03}$ & $\num{3.31e+04}$ & $\num{3.38e+04}$ \\
\datasettitle{\dmgam{}}
 & lin & $\num{78.22}$ & $\num{78.26}$ & $\num{1.19e+04}$ & $\num{1.04e+04}$ & $\num{1.68e+05}$ & $\num{7.03e+04}$ \\
 & rbf & $\num{87.94}$ & $\num{87.98}$ & $\num{8.25e+03}$ & $\num{7.46e+03}$ & $\num{3.10e+04}$ & $\num{3.06e+04}$ \\

  \bottomrule
 \end{tabular}
\end{mtable}

It should be noticed that, for these larger datasets, only one execution is done per dataset and kernel, and hence it is difficult to get solid conclusions.
Hence, it can be interesting to analyse the performance of the models during the CV phase, as done below.

\subsubsection{Robustness w.r.t. Hyper-Parameter \texorpdfstring{$C$}{C}}

The evolution with respect to the parameter $C$ of the accuracy, the number of support vectors and the number of training iterations is shown in \ref{FigResultsParIJCNN} for both \ssvm{} and the proposed \isvm{}. For the RBF kernel, the curves correspond to the optimum value of $\sigma$ for \ssvm{}.
Observing the plots of the accuracy, \isvm{} turns out to be much more stable than \ssvm{}, getting an accuracy almost optimal and larger than that of \ssvm{} in a wide range of values of $C$. Moreover, this accuracy is achieved with a smaller number of support vectors and with less training iterations. At some point, when the value of $C$ is large enough, both \ssvm{} and \isvm{} perform the same since all the support vectors of \ssvm{} also become working vectors during the training of \isvm{}, and both algorithms \fw{} and \mfw{} provide the same model.

{
 \tikzset{trim axis left, trim axis right}
 \pgfplotscreateplotcyclelist{mycolorlist}{%
graphic1,line width=0.25pt,pattern=horizontal lines,pattern color=graphic1, opacity=0.25\\%
graphic5,line width=0.25pt,pattern=north west lines,pattern color=graphic5, opacity=0.25\\%
graphic1,line width=1.0pt\\%
graphic5,line width=1.0pt,dashed\\%
 }
 \renewcommand{\plotline}[1]{%
  \tikzset{external/export next=false}%
  \begin{tikzpicture}[]
    \begin{axis}[xtick=\empty, hide axis, scale only axis,width=10pt, height=5pt, xmin=0, xmax=1, ymin=0, ymax=2, cycle list name=mycolorlist, cycle list shift=#1]
      \addplot
      coordinates {
      (0,1)
      (1,1)
      };
    \end{axis}
  \end{tikzpicture}}
 \renewcommand{\plotarea}[1]{%
  \tikzset{external/export next=false}%
  \begin{tikzpicture}
    \begin{axis}[xtick=\empty, hide axis, scale only axis,width=10pt, height=5pt, xmin=0, xmax=1, ymin=0, ymax=1, cycle list name=mycolorlist, cycle list shift=#1]
      \addplot
      coordinates {
      (0,1)
      (1,1)} \closedcycle;
    \end{axis}
  \end{tikzpicture}}
 \pgfplotsset{scale only axis, width=\textwidth, height=0.4\textwidth, clip=true, cycle list name=mycolorlist, xmin=0.00001, xmax=100000, ylabel near ticks, xlabel near ticks, xtick={0.001, 1, 1000}}
 \begin{mfigurel}{\label{FigResultsParIJCNN} Evolution of the validation results for \dijcn{} and \dmgam{}, using both the linear and the RBF kernel for the optimum $\sigma$ of \ssvm{}, both for the standard \ssvm{} and the proposed \isvm{}. The striped regions represent the range between minimum and maximum for the \num{10} partitions, whereas the lines in the middle represent the average values.}
  \legend{\showlegenddouble{1}{3}~\isvm{}}{\showlegenddouble{0}{2}~\ssvm{}\seplegend\showlegenddouble{1}{3}~\isvm{}}
  \tikzwidth{0.49\textwidth}%
  \pgfplotsset{yticklabel pos=left}%
  \subfloat[\label{FigResultsParIJCNNlin} Linear kernel for \dijcn{}.]{\shortstack[r]{%
   \pgfplotsset{xticklabels={,,}, ylabel={Accuracy ($\%$)}, ymin=6.92e+01, ymax=9.89e+01}\ifusetikz\begin{minipage}{\figwidth}\vspace{0pt}\scriptsize\tikzsetnextfilename{demoIncSVMDataset-ijcnn-lin-01A}{\tikzset{trim axis left,trim axis right}\input{./demoIncSVMDataset-ijcnn-lin-01A.tikz}}\end{minipage}\else\includegraphics{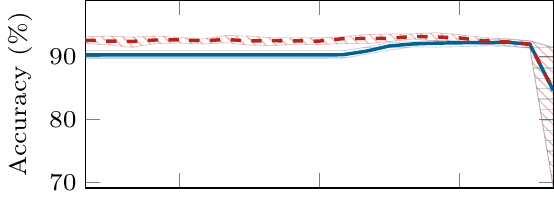}\fi\\
   \pgfplotsset{xticklabels={,,}, ylabel={Number of SVs}, ymin=4.42e+03, ymax=4.50e+04}\ifusetikz\begin{minipage}{\figwidth}\vspace{0pt}\scriptsize\tikzsetnextfilename{demoIncSVMDataset-ijcnn-lin-01N}{\tikzset{trim axis left,trim axis right}\input{./demoIncSVMDataset-ijcnn-lin-01N.tikz}}\end{minipage}\else\includegraphics{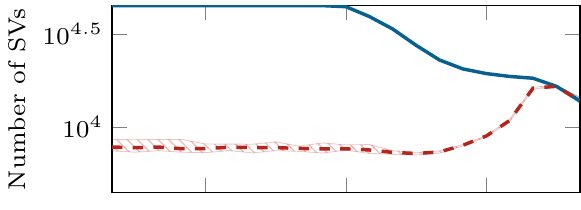}\fi\\
   \pgfplotsset{xlabel=$C$, ylabel={Iterations}, ymin=9.42e+03, ymax=4.02e+05}\ifusetikz\begin{minipage}{\figwidth}\vspace{0pt}\scriptsize\tikzsetnextfilename{demoIncSVMDataset-ijcnn-lin-01I}{\tikzset{trim axis left,trim axis right}\input{./demoIncSVMDataset-ijcnn-lin-01I.tikz}}\end{minipage}\else\includegraphics{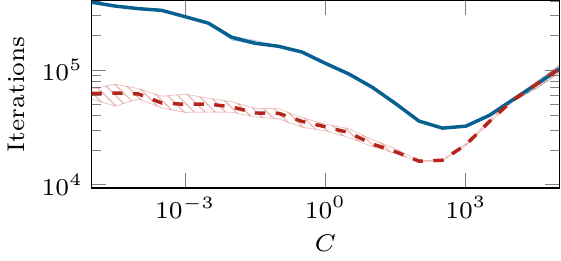}\fi%
   }}\ifusetikz\hfill\else\,\fi%
  \pgfplotsset{yticklabel pos=right}%
  \subfloat[\label{FigResultsParIJCNNrbf} RBF kernel for \dijcn{}.]{\shortstack[l]{%
   \pgfplotsset{xticklabels={,,} ,ylabel={Accuracy ($\%$)}, ymin=6.92e+01, ymax=9.89e+01}\ifusetikz\begin{minipage}{\figwidth}\vspace{0pt}\scriptsize\tikzsetnextfilename{demoIncSVMDataset-ijcnn-rbf-01A}{\tikzset{trim axis left,trim axis right}\input{./demoIncSVMDataset-ijcnn-rbf-01A.tikz}}\end{minipage}\else\includegraphics{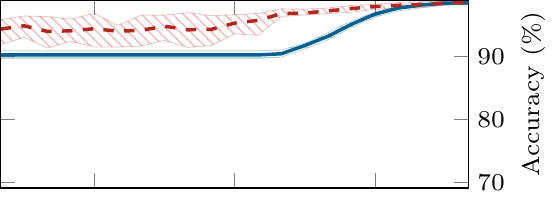}\fi\\
   \pgfplotsset{xticklabels={,,} ,ylabel={Number of SVs}, ymin=4.42e+03, ymax=4.50e+04}\ifusetikz\begin{minipage}{\figwidth}\vspace{0pt}\scriptsize\tikzsetnextfilename{demoIncSVMDataset-ijcnn-rbf-01N}{\tikzset{trim axis left,trim axis right}\input{./demoIncSVMDataset-ijcnn-rbf-01N.tikz}}\end{minipage}\else\includegraphics{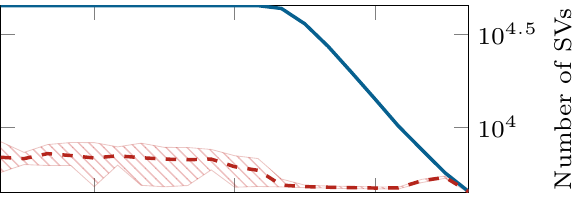}\fi\\
   \pgfplotsset{xlabel=$C$ ,ylabel={Iterations}, ymin=9.42e+03, ymax=4.02e+05}\ifusetikz\begin{minipage}{\figwidth}\vspace{0pt}\scriptsize\tikzsetnextfilename{demoIncSVMDataset-ijcnn-rbf-01I}{\tikzset{trim axis left,trim axis right}\input{./demoIncSVMDataset-ijcnn-rbf-01I.tikz}}\end{minipage}\else\includegraphics{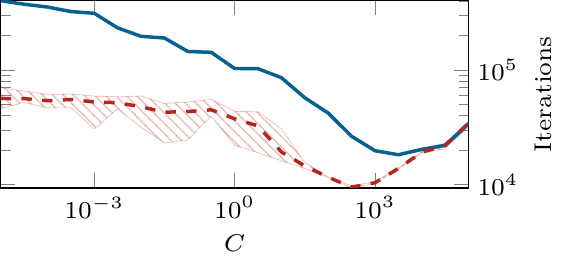}\fi%
   }}\\
  \pgfplotsset{yticklabel pos=left}%
  \subfloat[\label{FigResultsParMGAMMAlin} Linear kernel for \dmgam{}.]{\shortstack[r]{%
   \pgfplotsset{xticklabels={,,}, ylabel={Accuracy ($\%$)}, ymin=5.13e+01, ymax=8.86e+01}\ifusetikz\begin{minipage}{\figwidth}\vspace{0pt}\scriptsize\tikzsetnextfilename{demoIncSVMDataset-mgamma-lin-01A}{\tikzset{trim axis left,trim axis right}\input{./demoIncSVMDataset-mgamma-lin-01A.tikz}}\end{minipage}\else\includegraphics{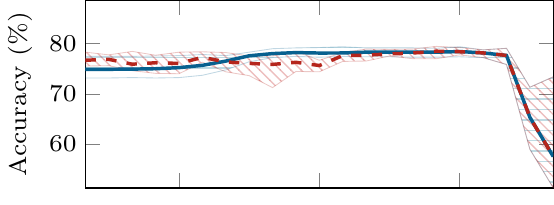}\fi\\
   \pgfplotsset{xticklabels={,,}, ylabel={Number of SVs}, ymin=3.88e+03, ymax=1.17e+04}\ifusetikz\begin{minipage}{\figwidth}\vspace{0pt}\scriptsize\tikzsetnextfilename{demoIncSVMDataset-mgamma-lin-01N}{\tikzset{trim axis left,trim axis right}\input{./demoIncSVMDataset-mgamma-lin-01N.tikz}}\end{minipage}\else\includegraphics{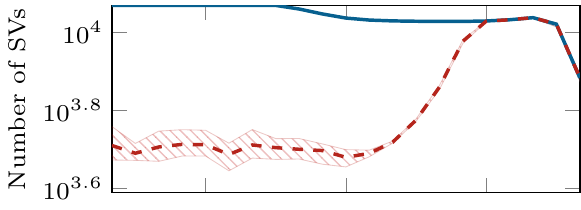}\fi\\
   \pgfplotsset{xlabel=$C$, ylabel={Iterations}, ymin=8.26e+03, ymax=1.52e+06}\ifusetikz\begin{minipage}{\figwidth}\vspace{0pt}\scriptsize\tikzsetnextfilename{demoIncSVMDataset-mgamma-lin-01I}{\tikzset{trim axis left,trim axis right}\input{./demoIncSVMDataset-mgamma-lin-01I.tikz}}\end{minipage}\else\includegraphics{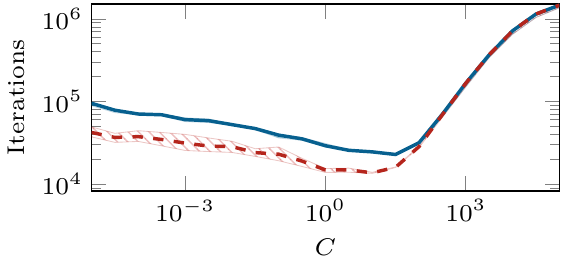}\fi%
   }}\ifusetikz\hfill\else\,\fi%
  \pgfplotsset{yticklabel pos=right}%
  \subfloat[\label{FigResultsParMGAMMArbf} RBF kernel for \dmgam{}.]{\shortstack[l]{%
   \pgfplotsset{xticklabels={,,}, ylabel={Accuracy ($\%$)}, ymin=5.13e+01, ymax=8.86e+01}\ifusetikz\begin{minipage}{\figwidth}\vspace{0pt}\scriptsize\tikzsetnextfilename{demoIncSVMDataset-mgamma-rbf-01A}{\tikzset{trim axis left,trim axis right}\input{./demoIncSVMDataset-mgamma-rbf-01A.tikz}}\end{minipage}\else\includegraphics{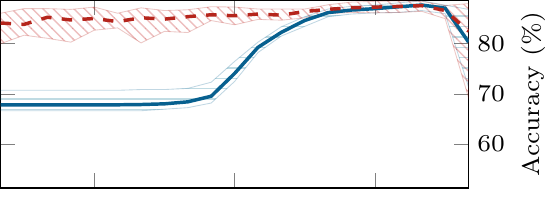}\fi\\
   \pgfplotsset{xticklabels={,,}, ylabel={Number of SVs}, ymin=3.88e+03, ymax=1.17e+04}\ifusetikz\begin{minipage}{\figwidth}\vspace{0pt}\scriptsize\tikzsetnextfilename{demoIncSVMDataset-mgamma-rbf-01N}{\tikzset{trim axis left,trim axis right}\input{./demoIncSVMDataset-mgamma-rbf-01N.tikz}}\end{minipage}\else\includegraphics{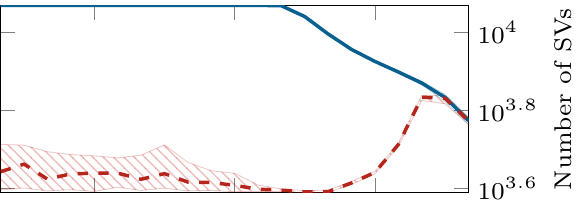}\fi\\
   \pgfplotsset{xlabel=$C$, ylabel={Iterations}, ymin=8.26e+03, ymax=1.52e+06}\ifusetikz\begin{minipage}{\figwidth}\vspace{0pt}\scriptsize\tikzsetnextfilename{demoIncSVMDataset-mgamma-rbf-01I}{\tikzset{trim axis left,trim axis right}\input{./demoIncSVMDataset-mgamma-rbf-01I.tikz}}\end{minipage}\else\includegraphics{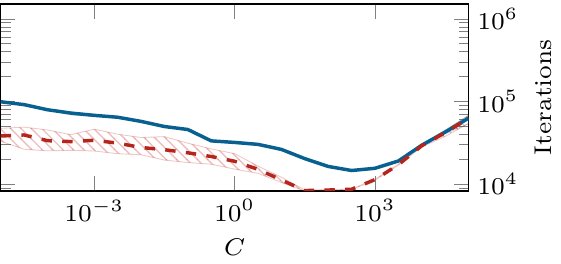}\fi%
   }}
 \end{mfigurel}
}

The stability of \isvm{} concerning the value of the regularization parameter suggests to fix $C$ beforehand in order to get rid of a tuning parameter. This option will be explored in the next bunch of experiments.

\subsection{Exhaustive Experiments}
\label{SecExpEx}

In the following experiments, the smaller \num{9} datasets of the second block of \ref{TabDatasets} will be used to compare exhaustively three models: \ssvm{}, the proposed \isvm{}, and an alternative \isvm{} model with a fixed regularization parameter (denoted as \fisvm{}), in particular $C = 1$ (normalized).

\subsubsection{Set-Up}

As in the previous experiments, the hyper-parameters will be obtained using \num{10}-fold CV (except for \fisvm{}, where $C$ is fixed and only $\sigma$ will be tuned for the RBF kernel). The stopping criterion is again $\epsilon = \num{e-5}$. Once trained, the models will be compared over the test set.

Furthermore, in order to study the significance of the differences between the models, the whole procedure, including the CV and the test phase, will be repeated \num{10} times for different training/test partitions of the data (with a proportion \SI{90}{\percent}/\SI{10}{\percent}).

\subsubsection{Results}

The results are detailed in \ref{TabResultsInc}, which includes for each of the three models the mean and standard deviation of the accuracy, the number of support vectors and the number of training iterations over the \num{10} partitions. The colours represent the rank of the models for each dataset and kernel, where the same rank is used if there is no significant difference between the models\footnote{Using a Wilcoxon signed rank test for zero median, with a significance level of $5\%$.}. 

\begin{mtablel}{\label{TabResultsInc} Test results for the exhaustive experiments (\num{10} repetitions). The colour indicates the rank (the darker, the better).}
\sisetup{output-exponent-marker=\textsc{e},exponent-product={},retain-explicit-plus}
\renewcommand{\widthbox}{\widthof{$\num{9.99e+9}\pm\num{9.9e+9}$}}
 \begin{tabular}{l@{}c@{\enspace}*3{c@{\,}}}
  \toprule
  \tabformathrow{Data} & \tabformathrow{K.} & \tabformathrow{\ssvm{}} & \tabformathrow{\isvm{}} & \tabformathrow{\fisvm{}} \\
  \midrule & & \multicolumn{3}{c}{\tabformathrow{Accuracy ($\%$)}} \\ \cmidrule{3-5}
  \datasettitle{\daust{}}
 & lin & \formata{$\num{85.65} \pm \num{4.4}$} & \formata{$\num{86.09} \pm \num{4.2}$} & \formata{$\num{85.65} \pm \num{4.1}$} \\
 & rbf & \formata{$\num{85.94} \pm \num{4.5}$} & \formata{$\num{85.36} \pm \num{4.1}$} & \formata{$\num{85.22} \pm \num{4.1}$} \\
\datasettitle{\dbrea{}}
 & lin & \formata{$\num{96.92} \pm \num{1.8}$} & \formata{$\num{96.49} \pm \num{1.7}$} & \formata{$\num{96.49} \pm \num{2.1}$} \\
 & rbf & \formata{$\num{96.63} \pm \num{2.1}$} & \formata{$\num{96.49} \pm \num{1.9}$} & \formata{$\num{96.78} \pm \num{1.7}$} \\
\datasettitle{\ddiab{}}
 & lin & \formatc{$\num{77.35} \pm \num{3.9}$} & \formata{$\num{78.52} \pm \num{3.0}$} & \formatc{$\num{76.57} \pm \num{4.4}$} \\
 & rbf & \formata{$\num{77.48} \pm \num{3.2}$} & \formata{$\num{77.09} \pm \num{3.3}$} & \formata{$\num{75.92} \pm \num{4.6}$} \\
\datasettitle{\dgerm{}}
 & lin & \formata{$\num{76.70} \pm \num{3.3}$} & \formata{$\num{76.60} \pm \num{4.1}$} & \formata{$\num{76.70} \pm \num{4.8}$} \\
 & rbf & \formata{$\num{76.70} \pm \num{5.2}$} & \formata{$\num{76.30} \pm \num{4.2}$} & \formata{$\num{76.00} \pm \num{4.9}$} \\
\datasettitle{\dhear{}}
 & lin & \formata{$\num{82.59} \pm \num{6.5}$} & \formata{$\num{83.33} \pm \num{7.3}$} & \formata{$\num{84.44} \pm \num{7.2}$} \\
 & rbf & \formata{$\num{83.33} \pm \num{8.2}$} & \formata{$\num{83.33} \pm \num{6.6}$} & \formata{$\num{84.44} \pm \num{7.6}$} \\
\datasettitle{\diono{}}
 & lin & \formata{$\num{82.65} \pm \num{6.9}$} & \formata{$\num{82.65} \pm \num{6.9}$} & \formata{$\num{81.79} \pm \num{7.7}$} \\
 & rbf & \formata{$\num{92.61} \pm \num{6.4}$} & \formata{$\num{91.19} \pm \num{6.3}$} & \formata{$\num{92.03} \pm \num{5.7}$} \\
\datasettitle{\diris{}}
 & lin & \formata{$\num{100.00} \pm \num{0.0}$} & \formata{$\num{100.00} \pm \num{0.0}$} & \formata{$\num{100.00} \pm \num{0.0}$} \\
 & rbf & \formata{$\num{100.00} \pm \num{0.0}$} & \formata{$\num{99.33} \pm \num{2.1}$} & \formata{$\num{99.33} \pm \num{2.1}$} \\
\datasettitle{\dmush{}}
 & lin & \formata{$\num{100.00} \pm \num{0.0}$} & \formata{$\num{100.00} \pm \num{0.0}$} & \formata{$\num{100.00} \pm \num{0.0}$} \\
 & rbf & \formata{$\num{100.00} \pm \num{0.0}$} & \formata{$\num{100.00} \pm \num{0.0}$} & \formata{$\num{100.00} \pm \num{0.0}$} \\
\datasettitle{\dsona{}}
 & lin & \formata{$\num{72.60} \pm \num{7.3}$} & \formata{$\num{70.21} \pm \num{8.4}$} & \formata{$\num{71.67} \pm \num{10.5}$} \\
 & rbf & \formata{$\num{87.00} \pm \num{4.6}$} & \formata{$\num{88.50} \pm \num{6.4}$} & \formata{$\num{87.00} \pm \num{5.1}$} \\

  \midrule & & \multicolumn{3}{c}{\tabformathrow{Number SVs}} \\ \cmidrule{3-5}
  \datasettitle{\daust{}}
 & lin & \formatc{$\num{5.94e+02} \pm \num{5.8e+01}$} & \formatb{$\num{4.40e+02} \pm \num{8.6e+01}$} & \formata{$\num{2.01e+02} \pm \num{6.6e+00}$} \\
 & rbf & \formatc{$\num{5.14e+02} \pm \num{7.8e+01}$} & \formatc{$\num{4.45e+02} \pm \num{9.0e+01}$} & \formata{$\num{2.14e+02} \pm \num{5.1e+01}$} \\
\datasettitle{\dbrea{}}
 & lin & \formatc{$\num{1.35e+02} \pm \num{1.5e+01}$} & \formatb{$\num{7.23e+01} \pm \num{1.9e+01}$} & \formata{$\num{5.52e+01} \pm \num{2.4e+00}$} \\
 & rbf & \formatc{$\num{2.93e+02} \pm \num{1.1e+02}$} & \formata{$\num{5.56e+01} \pm \num{1.7e+01}$} & \formata{$\num{5.62e+01} \pm \num{9.4e+00}$} \\
\datasettitle{\ddiab{}}
 & lin & \formatc{$\num{6.37e+02} \pm \num{2.8e+01}$} & \formatb{$\num{5.01e+02} \pm \num{1.1e+02}$} & \formata{$\num{3.61e+02} \pm \num{8.1e+00}$} \\
 & rbf & \formatc{$\num{6.06e+02} \pm \num{2.1e+01}$} & \formatb{$\num{4.94e+02} \pm \num{1.2e+02}$} & \formata{$\num{3.73e+02} \pm \num{2.1e+01}$} \\
\datasettitle{\dgerm{}}
 & lin & \formatc{$\num{8.04e+02} \pm \num{1.5e+01}$} & \formatb{$\num{6.18e+02} \pm \num{1.2e+02}$} & \formata{$\num{5.21e+02} \pm \num{9.5e+00}$} \\
 & rbf & \formatc{$\num{7.88e+02} \pm \num{3.8e+01}$} & \formatb{$\num{7.01e+02} \pm \num{1.1e+02}$} & \formata{$\num{4.82e+02} \pm \num{2.4e+01}$} \\
\datasettitle{\dhear{}}
 & lin & \formatc{$\num{2.18e+02} \pm \num{2.8e+01}$} & \formata{$\num{1.01e+02} \pm \num{4.1e+00}$} & \formata{$\num{1.02e+02} \pm \num{4.8e+00}$} \\
 & rbf & \formatc{$\num{1.92e+02} \pm \num{3.1e+01}$} & \formata{$\num{1.27e+02} \pm \num{1.6e+01}$} & \formata{$\num{1.28e+02} \pm \num{2.4e+01}$} \\
\datasettitle{\diono{}}
 & lin & \formatc{$\num{2.10e+02} \pm \num{1.9e+01}$} & \formatc{$\num{2.02e+02} \pm \num{2.7e+01}$} & \formata{$\num{1.40e+02} \pm \num{7.4e+00}$} \\
 & rbf & \formatc{$\num{1.72e+02} \pm \num{3.4e+01}$} & \formata{$\num{8.19e+01} \pm \num{2.3e+01}$} & \formata{$\num{7.20e+01} \pm \num{7.4e+00}$} \\
\datasettitle{\diris{}}
 & lin & \formatc{$\num{9.91e+01} \pm \num{1.8e+01}$} & \formata{$\num{2.60e+00} \pm \num{1.9e+00}$} & \formata{$\num{2.60e+00} \pm \num{1.9e+00}$} \\
 & rbf & \formatc{$\num{1.35e+02} \pm \num{0.0e+00}$} & \formata{$\num{2.70e+00} \pm \num{4.8e-01}$} & \formata{$\num{2.70e+00} \pm \num{4.8e-01}$} \\
\datasettitle{\dmush{}}
 & lin & \formatc{$\num{8.45e+02} \pm \num{1.5e+02}$} & \formata{$\num{1.12e+02} \pm \num{2.1e+01}$} & \formatb{$\num{1.40e+02} \pm \num{8.7e+00}$} \\
 & rbf & \formatc{$\num{7.31e+03} \pm \num{5.2e-01}$} & \formata{$\num{2.74e+01} \pm \num{2.0e+00}$} & \formata{$\num{2.72e+01} \pm \num{1.8e+00}$} \\
\datasettitle{\dsona{}}
 & lin & \formata{$\num{1.04e+02} \pm \num{3.6e+01}$} & \formata{$\num{1.12e+02} \pm \num{2.5e+01}$} & \formatc{$\num{1.36e+02} \pm \num{3.7e+00}$} \\
 & rbf & \formatc{$\num{1.44e+02} \pm \num{2.9e+01}$} & \formata{$\num{6.50e+01} \pm \num{1.0e+01}$} & \formata{$\num{7.10e+01} \pm \num{8.6e+00}$} \\

  \midrule & & \multicolumn{3}{c}{\tabformathrow{Number Iters.}} \\ \cmidrule{3-5}
  \datasettitle{\daust{}}
 & lin & \formatb{$\num{2.65e+04} \pm \num{5.6e+04}$} & \formatc{$\num{9.98e+04} \pm \num{9.9e+04}$} & \formata{$\num{6.41e+02} \pm \num{1.6e+01}$} \\
 & rbf & \formatc{$\num{1.26e+04} \pm \num{1.2e+04}$} & \formatc{$\num{1.74e+04} \pm \num{1.4e+04}$} & \formata{$\num{7.40e+02} \pm \num{1.2e+02}$} \\
\datasettitle{\dbrea{}}
 & lin & \formatc{$\num{2.06e+04} \pm \num{6.1e+04}$} & \formatb{$\num{8.93e+02} \pm \num{7.6e+02}$} & \formata{$\num{1.90e+02} \pm \num{1.0e+01}$} \\
 & rbf & \formatc{$\num{3.09e+03} \pm \num{7.5e+03}$} & \formata{$\num{2.56e+03} \pm \num{7.3e+03}$} & \formata{$\num{2.29e+02} \pm \num{3.8e+01}$} \\
\datasettitle{\ddiab{}}
 & lin & \formatc{$\num{1.42e+04} \pm \num{3.1e+04}$} & \formatc{$\num{9.35e+03} \pm \num{1.6e+04}$} & \formata{$\num{1.16e+03} \pm \num{5.9e+01}$} \\
 & rbf & \formatc{$\num{1.11e+04} \pm \num{9.6e+03}$} & \formatc{$\num{9.33e+03} \pm \num{1.0e+04}$} & \formata{$\num{1.39e+03} \pm \num{1.8e+02}$} \\
\datasettitle{\dgerm{}}
 & lin & \formatc{$\num{7.51e+04} \pm \num{1.6e+05}$} & \formatc{$\num{6.51e+04} \pm \num{1.6e+05}$} & \formata{$\num{2.02e+03} \pm \num{3.2e+01}$} \\
 & rbf & \formatc{$\num{7.90e+03} \pm \num{8.4e+03}$} & \formatc{$\num{7.82e+03} \pm \num{5.2e+03}$} & \formata{$\num{1.38e+03} \pm \num{7.5e+01}$} \\
\datasettitle{\dhear{}}
 & lin & \formatc{$\num{3.14e+04} \pm \num{9.6e+04}$} & \formatb{$\num{5.81e+02} \pm \num{1.5e+02}$} & \formata{$\num{3.80e+02} \pm \num{1.9e+01}$} \\
 & rbf & \formatc{$\num{1.66e+04} \pm \num{1.6e+04}$} & \formatb{$\num{5.26e+02} \pm \num{1.7e+02}$} & \formata{$\num{3.90e+02} \pm \num{7.7e+01}$} \\
\datasettitle{\diono{}}
 & lin & \formatc{$\num{9.98e+04} \pm \num{1.0e+05}$} & \formatc{$\num{9.36e+04} \pm \num{1.1e+05}$} & \formata{$\num{8.91e+02} \pm \num{4.8e+01}$} \\
 & rbf & \formatc{$\num{1.24e+03} \pm \num{6.2e+02}$} & \formatb{$\num{8.07e+02} \pm \num{8.0e+02}$} & \formata{$\num{2.16e+02} \pm \num{2.8e+01}$} \\
\datasettitle{\diris{}}
 & lin & \formatc{$\num{2.63e+02} \pm \num{6.3e+01}$} & \formata{$\num{5.40e+00} \pm \num{1.1e+01}$} & \formata{$\num{5.40e+00} \pm \num{1.1e+01}$} \\
 & rbf & \formatc{$\num{1.22e+03} \pm \num{0.0e+00}$} & \formatb{$\num{2.86e+01} \pm \num{1.8e+01}$} & \formata{$\num{1.67e+01} \pm \num{1.0e+01}$} \\
\datasettitle{\dmush{}}
 & lin & \formatc{$\num{7.49e+03} \pm \num{1.6e+03}$} & \formatb{$\num{6.98e+02} \pm \num{1.3e+02}$} & \formata{$\num{5.21e+02} \pm \num{2.6e+01}$} \\
 & rbf & \formatc{$\num{5.61e+04} \pm \num{3.1e+03}$} & \formatb{$\num{2.85e+02} \pm \num{2.7e+01}$} & \formata{$\num{1.61e+02} \pm \num{1.4e+01}$} \\
\datasettitle{\dsona{}}
 & lin & \formatc{$\num{4.36e+04} \pm \num{4.2e+04}$} & \formatc{$\num{2.22e+04} \pm \num{3.2e+04}$} & \formata{$\num{2.66e+03} \pm \num{1.1e+02}$} \\
 & rbf & \formatc{$\num{9.67e+02} \pm \num{6.0e+02}$} & \formata{$\num{2.65e+02} \pm \num{1.7e+02}$} & \formata{$\num{1.96e+02} \pm \num{2.3e+01}$} \\

  \bottomrule
 \end{tabular}
\end{mtablel}

The results are averaged as a summary in \ref{TabResultsPercentage}, where they are included as a percentage with respect to the reference \ssvm{}.
This table shows that \isvm{} allows to reduce the number of support vectors, and of training iterations, to a \SI{30.1}{\percent} and a \SI{26.5}{\percent}, whereas the accuracy only drops to a \SI{99.8}{\percent}.
Moreover, using the \fisvm{} approach allows to avoid tuning $C$, while reducing the support vectors and iterations to a \SI{26.0}{\percent} and a \SI{8.0}{\percent}, with a drop of the accuracy to only the \SI{99.7}{\percent} of the \ssvm{} accuracy.

\begin{mtable}{\label{TabResultsPercentage} Geometric mean of the test results as a percentage with respect to \ssvm{} for the exhaustive experiments.}
\sisetup{detect-all}
 \begin{tabular}{l *3{S[table-format=3.2]}}
  \toprule
  & \tabformathrow{\ssvm{}} & \tabformathrow{\isvm{}} & \tabformathrow{\fisvm{}} \\
  \midrule
   \tabformathrow{Accuracy} & 100.00 & 99.80 & 99.67\\
\tabformathrow{Number SVs} & 100.00 & 30.12 & 25.95\\
\tabformathrow{Number Iters.} & 100.00 & 26.45 & 7.98\\
  \bottomrule
 \end{tabular}
\end{mtable}

\subsection{Evolution over a Large Dataset}

This section shows the evolution of the training algorithms over a larger dataset, namely the \dmini{} shown in \ref{TabDatasets}, for the three approaches \ssvm{}, \isvm{}, and \fisvm{}.

\subsubsection{Set-Up}

In this experiment the only kernel used is the RBF one. In order to set the hyper-parameters $C$ and $\sigma$, \num{10}-fold CV is applied over a small subsample of \num{5000} patterns. Although this approach can seem quite simplistic, it provides good enough parameters for the convergence comparison that is the goal of this experiment. In the case of \fisvm{}, $C$ is fixed as $C = 1$, and the optimal $\sigma$ of \ssvm{} is directly used instead of tuning it, so that no validation is done for this model.

Once $C$ and $\sigma$ are selected, the models are trained over the whole training set during \num{40000} iterations. During this process, intermediate models are extracted every \num{5000} iterations, simulating different selections of the stopping criterion $\epsilon$. These intermediate models (trained using \num{5000}, \num{10000}, \num{15000}... iterations) are used to predict over the test set, and thus they allow to analyse the evolution of the test accuracy as a function of the number of training iterations.

\subsubsection{Results}

The results are shown in \ref{FigResultsLarge}, which includes the evolution of the number of support vectors and the test accuracy.

{
 \tikzset{trim axis left, trim axis right}
 \pgfplotscreateplotcyclelist{mycolorlist}{%
graphic1,line width=1.0pt\\%
graphic5,line width=1.0pt,dashed\\%
graphic7,line width=1.0pt,dotted\\%
graphic1,line width=1.0pt,mark=x,mark options={solid}\\%
graphic5,line width=1.0pt,dashed,mark=x,mark options={solid}\\%
graphic7,line width=1.0pt,dotted,mark=x,mark options={solid}\\%
 }
 \renewcommand{\plotline}[1]{%
  \tikzset{external/export next=false}%
  \begin{tikzpicture}[]
    \begin{axis}[hide axis, scale only axis,width=10pt, height=5pt, xmin=0, xmax=1, ymin=0, ymax=2, cycle list name=mycolorlist, cycle list shift=#1]
      \addplot
      coordinates {
      (0,1)
      (1,1)
      };
    \end{axis}
  \end{tikzpicture}}

 \pgfplotsset{scale only axis, width=\textwidth, height=0.4\textwidth, clip=true, cycle list name=mycolorlist, xmin=0, xmax=4}
 \pgfplotsset{yticklabel pos=left}
 \begin{mfigure}{\label{FigResultsLarge} Evolution of the training for \dmini{} with RBF kernel, for the standard \ssvm{}, the proposed \isvm{} and the parameter free \fisvm{}. The accuracy corresponds to the test set.}
  \legend{\showlegend{2}~\fisvm{}}{\showlegend{0}~\ssvm{}\seplegend\showlegend{1}~\isvm{}\seplegend\showlegend{2}~\fisvm{}}
  \tikzwidth{0.66\textwidth}%
  \shortstack[r]{%
   \pgfplotsset{xticklabels={,,},scaled x ticks=false,ylabel={Accuracy ($\%$)},ymin=0,ymax=100,cycle list shift=3}\ifusetikz\begin{minipage}{\figwidth}\vspace{0pt}\scriptsize\tikzsetnextfilename{demoIncSVMLarge2-A}{\tikzset{trim axis left,trim axis right}\input{./demoIncSVMLarge2-A.tikz}}\end{minipage}\else\includegraphics{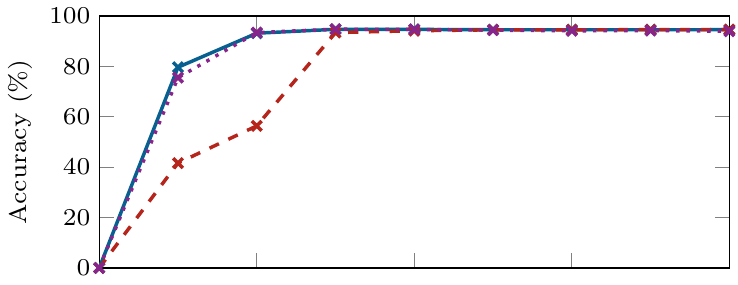}\fi\\
   \pgfplotsset{xlabel=Iteration (\num{e4}),ylabel={Number of SVs (\num{e4})},ymin=0,ymax=4,each nth point=100,filter discard warning=false,unbounded coords=discard}\ifusetikz\begin{minipage}{\figwidth}\vspace{0pt}\scriptsize\tikzsetnextfilename{demoIncSVMLarge2-N}{\tikzset{trim axis left,trim axis right}\input{./demoIncSVMLarge2-N.tikz}}\end{minipage}\else\includegraphics{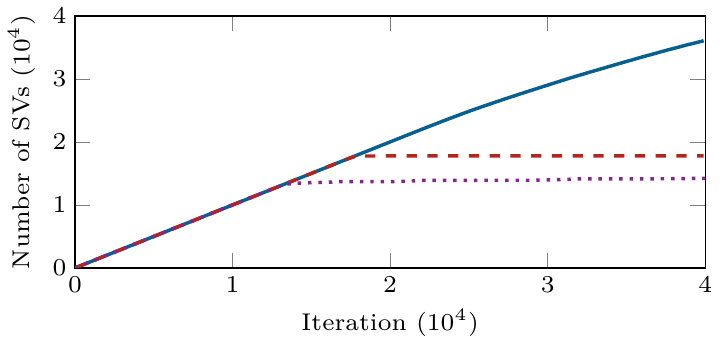}\fi%
   }
 \end{mfigure}
}

It can be observed that the standard \ssvm{} starts with the higher accuracy, but it is rapidly matched by \fisvm{}, and later by \isvm{}. Nevertheless, all of the models get finally a comparable and stable accuracy, and they reach it at approximately the same number of iterations (around \num{15000}).

The main difference can be seen in the evolution of the number of support vectors. In the first iterations, all the models introduce a new support vector at each iteration, but first \fisvm{} and second \isvm{} saturate this number presenting a final almost flat phase. On the contrary, although \ssvm{} reduces slightly the rate of growth of the number of support vectors, it continues adding more patterns to the solution during the whole training. This means that, if the stopping criterion is not carefully chosen for \ssvm{}, this model will use much more support vectors than needed, with the corresponding increase in its complexity. On the other side, \isvm{} and \fisvm{} (both models trained with \mfw{}) limit successfully the number of support vectors, providing sparser models with the same accuracy as \ssvm{}.

As a remark, it should be noticed that for \fisvm{} no validation phase was needed, since $C$ is fixed beforehand, and for $\sigma$ the optimal of \ssvm{} was used. 
This suggests again that \fisvm{} can be applied successfully with $C = 1$ and only tuning $\sigma$ if the RBF kernel is to be used.

\subsection{Dependence on the Initialization}

Another aspect of the proposed algorithm is its dependence on the initialization. Whereas the standard SVM is trained by solving a convex optimization problem with unique solution in the non-degenerate case, the proposed method summarized in \ref{AlgMFW} starts with an initial working vector that influences the resulting model, since it will determine the final subset of working vectors $\acto$.

\subsubsection{Set-Up}

A comparison of the models obtained using different initial working vectors will be done to study the variability due to the initialization.
In particular, for all \num{9} smaller datasets of \ref{TabDatasets} and in this case only for the linear kernel with the parameters obtained in \ref{SecExpEx} (no CV process is repeated), one model per possible initial point will be trained, so that at the end there will be as many models as training patterns for each partition.

\subsubsection{Results}

A first measure for the dependence on the initialization are the differences between the sets of support vectors of the models. \Ref{TabInitialOverlap} shows in the second column the average overlap between these sets of support vectors for every pair of models with different initializations, quantified as the percentage of support vectors that are shared on both models over the total number of support vectors\footnote{In particular, there are $\sfrac{\npat \prn{\npat - 1}}{2}$ measures per each one of the \num{10} repetitions, since there are $\npat$ different possible initializations (as many as training patterns).}. The two easiest datasets, \diris{} and \dmush{}, show the smallest overlaps (around \SI{30}{\percent}) and hence the highest dependence on the initialization. This is not surprising, since for example in the \diris{} dataset there are many hyperplanes that separate both classes perfectly. The remaining datasets show an overlap above \SI{80}{\percent}, and there are \num{4} datasets above \SI{95}{\percent}. Therefore, the influence on the initialization will depend strongly on the particular dataset.

\renewcommand{\datasettitle}[1]{\truncate[]{\sizetitle}{#1}}
\begin{mtable}{\label{TabInitialOverlap} Results for the initialization dependence, including the overlap of the different sets of support vectors for \isvm{}, and the accuracies of \ssvm{}, \isvm{} and \isvm{} considering all possible initializations.}
\sisetup{detect-all}
 \begin{tabular}{l*4{c}}
  \toprule
  \multirow{2}{*}{\tabformathrow{Data}} & \tabformathrow{SVs Overlap ($\%$)} & \multicolumn{3}{c}{\tabformathrow{Accuracy ($\%$)}} \\
  \cmidrule(lr){2-2} \cmidrule(lr){3-5}
  & \tabformathrow{\isvm{} Ini.} & \tabformathrow{\ssvm{}} & \tabformathrow{\isvm{}} & \tabformathrow{\isvm{} Ini.} \\
  \midrule
   \datasettitle{\daust{}} & $\num{95.69} \pm \num{4.9}$ & $\num{85.65} \pm \num{4.4}$ & $\num{86.09} \pm \num{4.2}$ & $\num{86.06} \pm \num{3.9}$ \\
\datasettitle{\dbrea{}} & $\num{83.96} \pm \num{4.2}$ & $\num{96.92} \pm \num{1.8}$ & $\num{96.49} \pm \num{1.7}$ & $\num{96.45} \pm \num{1.7}$ \\
\datasettitle{\ddiab{}} & $\num{97.26} \pm \num{2.2}$ & $\num{77.35} \pm \num{3.9}$ & $\num{78.52} \pm \num{3.0}$ & $\num{77.69} \pm \num{3.6}$ \\
\datasettitle{\dgerm{}} & $\num{92.24} \pm \num{4.8}$ & $\num{76.70} \pm \num{3.3}$ & $\num{76.60} \pm \num{4.1}$ & $\num{76.86} \pm \num{3.7}$ \\
\datasettitle{\dhear{}} & $\num{81.44} \pm \num{3.1}$ & $\num{82.59} \pm \num{6.5}$ & $\num{83.33} \pm \num{7.3}$ & $\num{82.87} \pm \num{7.6}$ \\
\datasettitle{\diono{}} & $\num{97.66} \pm \num{3.5}$ & $\num{82.65} \pm \num{6.9}$ & $\num{82.65} \pm \num{6.9}$ & $\num{83.16} \pm \num{6.4}$ \\
\datasettitle{\diris{}} & $\num{30.99} \pm \num{25.0}$ & $\num{100.00} \pm \num{0.0}$ & $\num{100.00} \pm \num{0.0}$ & $\num{99.66} \pm \num{1.5}$ \\
\datasettitle{\dmush{}} & $\num{32.61} \pm \num{5.2}$ & $\num{100.00} \pm \num{0.0}$ & $\num{100.00} \pm \num{0.0}$ & $\num{100.00} \pm \num{0.0}$ \\
\datasettitle{\dsona{}} & $\num{96.69} \pm \num{2.3}$ & $\num{72.60} \pm \num{7.3}$ & $\num{70.21} \pm \num{8.4}$ & $\num{70.47} \pm \num{8.7}$ \\

  \bottomrule
 \end{tabular}
\end{mtable}

Nevertheless, looking at the accuracies included in \ref{TabInitialOverlap}, and specifically comparing the results of \isvm{} when considering only one or all the possible initializations (columns \num{4} and \num{5}), it seems that there is no noticeable difference between them. In particular, and reducing the table to a single measure, the average error is \SI{86.05}{\percent} for \ssvm{}, \SI{85.99}{\percent} for \isvm{} and \SI{85.91}{\percent} for \isvm{} considering all the initializations.

Moreover, as an additional experiment \ref{FigResultsParInitialDependence} shows the results of an extra \num{10}-fold CV for the \dhear{} dataset with linear kernel, including the results of \isvm{} with all the possible initializations. It can observed that \isvm{} performs basically the same in average when changing the initial vector, in terms of all three the accuracy, the number of support vectors and the number of iterations, although obviously the distance between minimum and maximum value for each $C$ (striped region in the plots) increases since more experiments are included.

{
 \tikzset{trim axis left, trim axis right}
 \pgfplotscreateplotcyclelist{mycolorlist}{%
graphic1,line width=0.25pt,pattern=horizontal lines,pattern color=graphic1, opacity=0.25\\%
graphic5,line width=0.25pt,pattern=north west lines,pattern color=graphic5, opacity=0.25\\%
OliveGreen,line width=0.25pt,pattern=north east lines,pattern color=OliveGreen, opacity=0.25\\%
graphic1,line width=1.0pt\\%
graphic5,line width=1.0pt,dashed\\%
OliveGreen,line width=1.0pt,dotted\\%
 }
 \renewcommand{\plotline}[1]{%
  \tikzset{external/export next=false}%
  \begin{tikzpicture}[]
    \begin{axis}[hide axis, scale only axis,width=10pt, height=5pt, xmin=0, xmax=1, ymin=0, ymax=2, cycle list name=mycolorlist, cycle list shift=#1]
      \addplot
      coordinates {
      (0,1)
      (1,1)
      };
    \end{axis}
  \end{tikzpicture}}
 \renewcommand{\plotarea}[1]{%
  \tikzset{external/export next=false}%
  \begin{tikzpicture}
    \begin{axis}[hide axis, scale only axis,width=10pt, height=5pt, xmin=0, xmax=1, ymin=0, ymax=1, cycle list name=mycolorlist, cycle list shift=#1]
      \addplot
      coordinates {
      (0,1)
      (1,1)} \closedcycle;
    \end{axis}
  \end{tikzpicture}}

 \pgfplotsset{scale only axis, width=\textwidth, height=0.4\textwidth, clip=true, cycle list name=mycolorlist, xmin=0.00001, xmax=100000}
 \pgfplotsset{yticklabel pos=left}
 \begin{mfigure}{\label{FigResultsParInitialDependence} Evolution of the validation results for \dhear{} with linear kernel, for the standard \ssvm{}, the proposed \isvm{} and \isvm{} considering all possible initializations. The striped regions represent the range between minimum and maximum for the \num{10} partitions (\num{10} times the number of training patterns when considering all the possible initializations), whereas the lines in the middle represent the average values.}
 \legend{\showlegenddouble{2}{5}~\isvm{} Ini.}{\showlegenddouble{0}{3}~\ssvm{}\seplegend\showlegenddouble{1}{4}~\isvm{}\seplegend\showlegenddouble{2}{5}~\isvm{} Ini.}
  \tikzwidth{0.66\textwidth}%
  \shortstack[r]{%
   \pgfplotsset{xticklabels={,,},ylabel={Accuracy ($\%$)}, ymin=1.67e+01, ymax=1.00e+02}\ifusetikz\begin{minipage}{\figwidth}\vspace{0pt}\scriptsize\tikzsetnextfilename{demoInitialDependence3-heart-lin-A}{\tikzset{trim axis left,trim axis right}\input{./demoInitialDependence3-heart-lin-A.tikz}}\end{minipage}\else\includegraphics{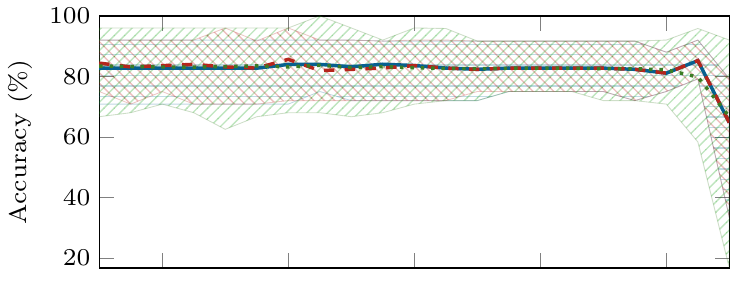}\fi\\
   \pgfplotsset{xticklabels={,,},ylabel={Number of SVs}, ymin=6.60e+01, ymax=2.19e+02}\ifusetikz\begin{minipage}{\figwidth}\vspace{0pt}\scriptsize\tikzsetnextfilename{demoInitialDependence3-heart-lin-N}{\tikzset{trim axis left,trim axis right}\input{./demoInitialDependence3-heart-lin-N.tikz}}\end{minipage}\else\includegraphics{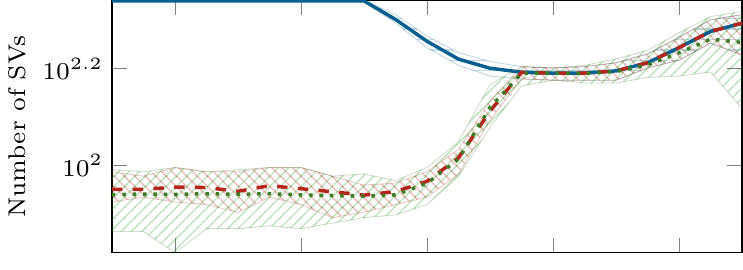}\fi\\
   \pgfplotsset{xlabel=$C$,ylabel={Iterations}, ymin=2.54e+02, ymax=4.95e+05}\ifusetikz\begin{minipage}{\figwidth}\vspace{0pt}\scriptsize\tikzsetnextfilename{demoInitialDependence3-heart-lin-I}{\tikzset{trim axis left,trim axis right}\input{./demoInitialDependence3-heart-lin-I.tikz}}\end{minipage}\else\includegraphics{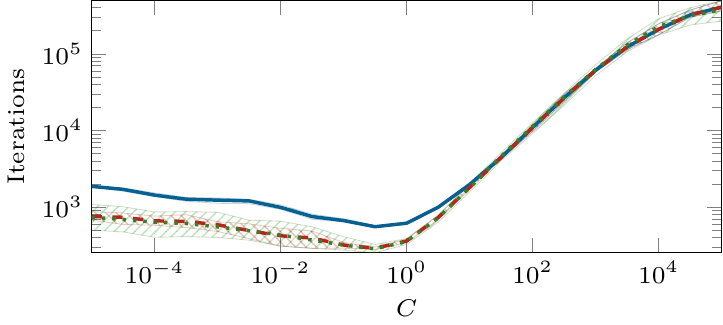}\fi%
   }
 \end{mfigure}
}

Therefore, it can be concluded that, although the proposed method can depend strongly on the initialization for some datasets, it seems that the resulting models are comparable in terms of accuracy, number of support vectors and required training iterations.
On the other side, it should be noticed that trying to establish a methodology to initialize in a clever way the algorithm would probably need of a considerable overhead, since the computational advantage of Frank--Wolfe and related methods is that they compute the gradient incrementally because the changes only affect a few coordinates. A comparison between all the possible initial vertices, leaving aside heuristics, would require the use of the whole kernel matrix, what could be prohibitive for large datasets.

\section{Conclusions}
\label{SecConc}

The connection between Lasso and Support Vector Machines (SVMs) has been used to propose an algorithmic improvement in the Frank--Wolfe (\fw{}) algorithm used to train the SVM. This modification is based on the re-weighted Lasso to enforce more sparsity, and computationally it just requires an additional conditional check at each iteration, so that the overall complexity of the algorithm remains the same.
The convergence analysis of this Modified Frank--Wolfe (\mfw{}) algorithm shows that it provides exactly the same SVM model that one would obtain applying the original \fw{} algorithm only over a subsample of the training set.
Several numerical experiments have shown that \mfw{} leads to models comparable in terms of accuracy, but with a sparser dual representation, requiring less iterations to be trained, and much more robust with respect to the regularization parameter, up to the extent of allowing to fix this parameter beforehand, thus avoiding its validation.

Possible lines of extension of this work are to explore other SVM formulations, for example based on the \lo{} loss, which should allow for even more sparsity. The \mfw{} algorithm could also be applied to the training of other machine learning algorithms such as non-negative Lasso, or even to general optimization problems that permit a certain relaxation of the original formulation.

\section*{Acknowledgments}

 The authors would like to thank the following organizations.
 \begin{itemize*}
  \item EU: The research leading to these results has received funding from the European Research Council under the European Union's Seventh Framework Programme (FP7/2007-2013) / ERC AdG A-DATADRIVE-B (290923). This paper reflects only the authors' views, the Union is not liable for any use that may be made of the contained information.
  \item Research Council KUL: GOA/10/09 MaNet, CoE PFV/10/002 (OPTEC), BIL12/11T; PhD/Postdoc grants.
  \item Flemish Government:
  \begin{itemize*}
   \item FWO: G.0377.12 (Structured systems), G.088114N (Tensor based data similarity); PhD/Postdoc grants.
   \item IWT: SBO POM (100031); PhD/Postdoc grants.
  \end{itemize*}
  \item iMinds Medical Information Technologies SBO 2014.
  \item Belgian Federal Science Policy Office: IUAP P7/19 (DYSCO, Dynamical systems, control and optimization, 2012-2017).
  \item Fundaci\'on BBVA: project FACIL--Ayudas Fundaci\'on BBVA a Equipos de Investigaci\'on Cient\'ifica 2016.
  \item UAM--ADIC Chair for Data Science and Machine Learning.
 \end{itemize*}

\section*{References}

\bibliographystyle{elsarticle-num} 
\bibliography{Bibliography}

\end{document}